\documentclass{ecai} 

\usepackage{latexsym}
\usepackage{amssymb}
\usepackage{amsmath}
\usepackage{amsthm}
\usepackage{booktabs}
\renewcommand\labelenumi{(\roman{enumi})}
\usepackage{graphicx}
\usepackage{color}
\usepackage{todonotes}

\usepackage{paralist}

\usepackage[nolist,nohyperlinks]{acronym}

\usepackage{pgf}
\usepackage{tikz}
\usetikzlibrary{arrows,automata}
\usetikzlibrary{shapes,snakes}
\usetikzlibrary{intersections}

\usepackage{mathtools,algorithm}
\usepackage[noend]{algpseudocode}

\usepackage[capitalise]{cleveref}

\newfloat{operator}{htbp}{lon}
\floatname{operator}{Operator}
\makeatletter\newcommand\l@operator{\@dottedtocline{1}{1.5em}{2.3em}}\makeatother

\tikzstyle{line} = [draw, -latex']
\tikzstyle{hidden} = [ellipse, draw, text centered, inner sep=1pt]
\tikzstyle{obs} = [ellipse, draw, fill=gray!60, text centered, inner sep=1pt]
\tikzstyle{rv} = [draw, ellipse, inner sep=2pt]
\tikzstyle{pf} = [draw, rectangle, fill=gray]
\tikzstyle{pc} = [draw, rounded corners=15pt, align=center, minimum width=18mm, font=\normalsize, inner sep=3pt]
\tikzstyle{pc2} = [draw, rounded corners=8pt,align=center,minimum height=6mm,font=\normalsize,inner sep=2pt]

\tikzstyle{nhidden} = [draw=none,fill=none,ellipse, text centered, inner sep=1pt]
\tikzstyle{nobs} = [draw=none,fill=none,ellipse, fill=gray!60, text centered, inner sep=1pt]
\tikzstyle{nrv} = [draw=none,fill=none,ellipse, inner sep=2pt]
\tikzstyle{npf} = [draw=none,fill=none,rectangle]

\tikzstyle{ID} = [draw, circle, font=\normalsize]
\tikzstyle{INN} = [draw, circle, inner sep=1pt, fill=black]

\newcommand\pfs[8]{
  \node[pf, #1 of=#2, node distance=#3, xshift=-1mm, yshift=1mm](#6) {};
  \node[pf, #1 of=#2, node distance=#3, label=#4:{#5}](#7) {};
  \node[pf, #1 of=#2, node distance=#3, xshift=1mm, yshift=-1mm](#8) {};
}

\newtheorem{theorem}{Theorem}
\newtheorem{lemma}[theorem]{Lemma}
\newtheorem{corollary}[theorem]{Corollary}

\newtheorem{definition}{Definition}
\newtheorem{example}{Example}
\theoremstyle{conjecture}
\newtheorem{conjecture}{Conjecture}

\newcommand{\BibTeX}{B\kern-.05em{\sc i\kern-.025em b}\kern-.08em\TeX}

\begin{document}


\begin{frontmatter}


\paperid{475} 


\title{On the Completeness and Complexity of Lifted Temporal Inference}


\author{\fnms{Marcel}~\snm{Gehrke}\orcid{0000-0001-9056-7673}\thanks{Corresponding Author. Email: marcel.gehrke@uni-hamburg.de.}}

\address{University of Hamburg}


	  \acrodef{SHR}[SHR]{Standard Health Record}
	  \acrodef{ehr}[EHR]{electronic health record}
	  \acrodef{FHIM}[FHIM]{Federal Health Information Model}
	  \acrodef{OHDSI}[OHDSI]{OMOP Common Data Model, the Observational Health Data Sciences and Informatics}
	  \acrodef{FHIR}[FHIR]{HL7’s FAST Healthcare Interoperability Resources}

	  \acrodef{jtree}[jtree]{junction tree}
	  \acrodef{plms}[PLMs]{probabilistic logical models}

	  \acrodef{pdb}[PDB]{probabilistic database}

	  \acrodef{pf}[parfactor]{parametric factor}
	  \acrodef{lv}[logical variable]{logical variable}
	  \acrodef{crv}[CRV]{counting random variable}
	  \acrodef{prv}[PRV]{parameterised random variable}

	  \acrodef{fodt}[FO dtree]{first-order decomposition tree}

	  \acrodef{fojt}[FO jtree]{first-order junction tree}
	  \acrodef{ljt}[LJT]{lifted junction tree algorithm}
	  \acrodef{ldjt}[LDJT]{lifted dynamic junction tree algorithm}
	  \acrodef{lve}[LVE]{lifted variable elimination}
	  \acrodef{ve}[VE]{variable elimination}

	  \acrodef{tam}[TAMe]{temporal approximate merging}

	  \acrodef{2tpm}[2TPM]{two-slice temporal parameterised model}
	  \acrodef{2tbn}[2TBN]{two-slice temporal bayesian network}

	  \acrodef{pm}[PM]{parameterised probabilistic model}
	  \acrodef{pdecm}[PDecM]{parameterised probabilistic decision model}

	  \acrodef{pdm}[PDM]{parameterised probabilistic dynamic model}
	  \acrodef{pddecm}[PDDecM]{parameterised probabilistic dynamic decision model}

	  \acrodef{dbn}[DBN]{dynamic Bayesian network}
	  \acrodef{bn}[BN]{Bayesian network}

	  \acrodef{dfg}[DFG]{dynamic factor graph}

	  \acrodef{dmln}[DMLN]{dynamic Markov logic network}
	  \acrodef{mln}[MLN]{Markov logic network}

	  \acrodef{rdbn}[RDBN]{relational dynamic Bayesian network}

	  \acrodef{meu}[MEU]{maximum expected utility}
	  \acrodef{mldn}[MLDN]{Markov logic decision network}

	  \acrodef{cep}[CEP]{complex event processing}

	  \acrodef{krp}[KRP]{keeping reasoning polynomial}

\begin{abstract}
	
	For static lifted inference algorithms, completeness, i.e., domain liftability, is extensively studied.
	However, so far no domain liftability results for temporal lifted inference algorithms exist. 
	In this paper, we close this gap.
	More precisely, we contribute the first completeness and complexity analysis for a temporal lifted algorithm, the so-called \ac{ldjt}, which is the only exact lifted temporal inference algorithm out there.
	To handle temporal aspects efficiently, \ac{ldjt} uses conditional independences to proceed in time, leading to restrictions w.r.t.\ elimination orders.
	We show that these restrictions influence the domain liftability results and show that one particular case while proceeding in time, has to be excluded from $FO^2$. 
	Additionally, for the complexity of \ac{ldjt}, we prove that the lifted width is in even more cases smaller than the corresponding treewidth in comparison to static inference.
\end{abstract}
\acresetall	

\acused{lv}
\end{frontmatter}

\section{Introduction}\label{sec:intro}
For static lifted inference algorithms, completeness, i.e., domain liftability, is extensively studied.
Static lifted inference algorithms such as weighted first-order model counting (WFOMC), \ac{lve}, or the \ac{ljt} are domain liftable for the $FO^2$ fragment, i.e., for all models of the $FO^2$ fragment they solve the inference problem in polynomial time w.r.t.\ domain sizes \citep{taghipour2013completeness}.
However, so far no domain liftability results for temporal lifted inference algorithms exist. 
Actually, to the best of our knowledge, for temporal lifted algorithms, no complexity and completeness results exist.
In this paper, we close this gap, by analysing the \ac{ldjt} the only exact temporal lifted inference algorithm out there.
Thus, we identify model classes for which temporal lifted query answering is guaranteed and give bounds for the query answering runtime.

Poole proposes \ac{lve} as an exact inference algorithm on relational models \cite{poole2003first}, which has been extended by 
\citeauthor{Braz07} by generalising lifted summing out \cite{Braz07} and by \citeauthor{milch2008lifted} by introduce counting to lift certain computations where lifted summing out does not apply \cite{milch2008lifted}. 
\citeauthor{TagFiDaBl13} extend \ac{lve} to its current form \cite{TagFiDaBl13}.
\citeauthor{taghipour2013completeness} introduce completeness results for \ac{lve} with generalised counting \cite{taghipour2013completeness}.
\citeauthor{Bra20} shows that the results also hold for \ac{ljt} while answering multiple queries efficiently \cite{Bra20}.
The completeness results also hold for WFOMC \citep{broeck2011completeness}.
For WFOMC, these results have been extended in recent years mostly to include counting into the logic-based WFOMC \citep{kazemi2016new,van2021faster,ijcai2023p801,dilkas2023synthesising,van2023lifted}.
Counting is by design included in LVE based approaches and therefore for LDJT, we focus on $FO^2$.
\citeauthor{taghipour2013first} present complexity results for \ac{lve} \cite{taghipour2013first}.
\citeauthor{Bra20} extends these results to the case of answering multiple queries and \ac{ljt}.
However, these approaches do not account for temporal aspects \cite{Bra20}.
There are approximate temporal relational algorithms \citep{ahmadi2013exploiting,geier2011approximate,papai2012slice}, unfortunately without theoretical bounds.
Therefore, for static lifted algorithms, completeness and complexity analyses exist, but not for temporal lifted algorithms.

We contribute completeness and complexity results for \ac{ldjt}, which uses (temporal) conditional independences to proceed in time.
Specifically, we analyse
\begin{enumerate}[(i)]
	\item the influence of using temporal conditional independences on the elimination order w.r.t.\ liftable models,
	\item the completeness of \ac{ldjt},  
	\item the complexity of \ac{ldjt}, and 
	\item how domain sizes influence the lifted width of \ac{ldjt} in comparison to the treewidth of the propositional interface algorithm~\citep{Murphy:2002:DBN}. 
\end{enumerate} 
We show that using temporal conditional independences for efficient temporal inference leads to restrictions w.r.t.\ elimination orders of inference algorithms.
Compared to known completeness results for static approaches, we show that these restrictions lead to an adjustment in the class of liftable models.
For the completeness of \ac{ldjt} and thereby temporal lifted inference using temporal conditional independences, we prove that the liftability class has to be adjusted and prove a sufficient adjustment.
In our complexity analysis, we show that compared to static inference, lifted temporal inference has even more advantages over the ground case.  
Increasing the number of random variables that directly influence the next time step, leads to increasing the treewidth, an exponential term in the complexity of the propositional interface algorithm~\citep{Murphy:2002:DBN}, while the lifted width of \ac{ldjt} may remain constant.
Overall, with our completeness and complexity analysis, we give theoretical guarantees and show how crucial lifting is for temporal inference. 


\begin{figure*}
    \centering
	\begin{minipage}[b]{.28\textwidth}
        \centering
        \scalebox{0.83}{
        \begin{tikzpicture}[rv/.style={draw, ellipse, inner sep=1pt, minimum width=1.5cm},pf/.style={draw, rectangle, fill=gray},label distance=0.2mm]
        	\node[rv, inner sep=0.5pt]					 								(S)	{$N$ };
            \pfs{left}{S}{20mm}{230}{$g^0$}{USa}{US}{USb}    
        	\node[rv, left of=US, node distance=20mm, inner sep=0.5pt]			(U)	{$M(X,Y)$};    
        	\node[rv, below of=S, node distance=10mm, inner sep=0.5pt]			(T1)	{$O(Y)$};    
            \pfs{left}{T1}{20mm}{230}{$g^1$}{ASa}{AS}{ASb}    
        	\node[rv, left of=AS, node distance=20mm, inner sep=0.5pt]			(A)	{$L(X)$};    
    
        	\draw (U) -- (US);
        	\draw (US) -- (S);
        	\draw (US) -- (T1);

        	\draw (A) -- (AS);
        	\draw (AS) -- (S);
        	\draw (AS) -- (T1);

        \end{tikzpicture}
        }
        \caption{Parfactor graph for $G^{ex}$}
        \label{fig:swe}	
    \end{minipage}\hfill
    \begin{minipage}[b]{.69\textwidth}
\centering
\scalebox{0.83}{
\begin{tikzpicture}[rv/.style={draw, ellipse, inner sep=1pt,minimum width=1.5cm},pf/.style={draw, rectangle, fill=gray},label distance=0.2mm]
	\node[rv]					 								(S)	{$N_{t-1}$};
    \pfs{left}{S}{20mm}{230}{$g^0_{t-1}$}{USa}{US}{USb}    
	\node[rv, left of=US, node distance=20mm, inner sep=0.5pt]			(U)	{$\mathbf{M_{t-1}(X,Y)}$};    
	\node[rv, below of=S, node distance=10mm]						(T1)	{$O_{t-1}(Y)$};    
    \pfs{left}{T1}{20mm}{230}{$g^1_{t-1}$}{ASa}{AS}{ASb}    
	\node[rv, left of=AS, node distance=20mm, inner sep=0.5pt]			(A)	{$\mathbf{L_{t-1}(X)}$};    
    
	\node[rv, right of = S, node distance=80mm]					 		(S1)	{$N_{t}$};
    \pfs{left}{S1}{20mm}{230}{$g^0_{t}$}{USa}{US1}{USb}    
	\node[rv, left of=US1, node distance=20mm, inner sep=0.5pt]			(U1)	{$M_{t}(X,Y)$};    
	\node[rv, below of=S1, node distance=10mm]						(T11)	{$O_{t}(Y)$};    
    \pfs{left}{T11}{20mm}{230}{$g^1_{t}$}{ASa}{AS1}{ASb}    
	\node[rv, left of=AS1, node distance=20mm, inner sep=0.5pt]			(A1)	{$L_{t}(X)$};

    
    \pfs{right}{S}{23mm}{315}{$g^i_t$}{UAa}{IU}{UAb}

    \path [-, bend left=45] (IU) edge node {} (A);
    \path [-, bend right=15] (IU) edge node {} (U);
	\draw (IU) -- (U1);
    
	\draw (U) -- (US);
	\draw (US) -- (S);
	\draw (US) -- (T1);

	\draw (A) -- (AS);
	\draw (AS) -- (S);
	\draw (AS) -- (T1);

	\draw (U1) -- (US1);
	\draw (US1) -- (S1);
	\draw (US1) -- (T11);

	\draw (A1) -- (AS1);
	\draw (AS1) -- (S1);
	\draw (AS1) -- (T11);


\end{tikzpicture}
}
\caption{$G_\rightarrow^{ex}$ the two-slice temporal parfactor graph for model $G^{ex}$}
\label{fig:TSPG}	
\end{minipage}
\end{figure*}

In the following, we recap \acp{pdm} as a formalism for specifying temporal probabilistic relational models and \ac{ldjt} for efficient query answering in \acp{pdm}. 
Then, we show a type of model, where using temporal conditional independences leads to groundings.
We use these insights to derive our completeness results for \ac{ldjt}. 
Further, we investigate the complexity of \ac{ldjt} and set it into relation to propositional temporal algorithms and \ac{ljt}. 
Lastly, we conclude.

\section{Preliminaries}

We shortly present all parts of \ac{ldjt} needed for understanding the completeness and complexity analysis and present them in detail in the appendix.
\ac{ldjt} uses \acp{pdm} as a representation \citep{GehMoBr20}, which in turn are based on \acp{pm} \citep{braun2018parameterised}.

\subsection{Parameterised Probabilistic Models}\label{pm}

\acp{pm} combine first-order logic with probabilistic models.
A \ac{pm} is a parameterised factor graph using \acp{prv} to capture underlying symmetries. 

\begin{definition}
	Let $\mathbf{R}$ be a set of random variable names, $\mathbf{L}$ a set of \ac{lv} names, $\Phi$ a set of factor names, and $\mathbf{D}$ a set of constants (universe).
	All sets are finite.
	Each logical variable $L$ has a domain $\mathcal{D}(L) \subseteq \mathbf{D}$.
	A \emph{constraint} is a tuple $(\mathcal{X}, C_{\mathbf{X}})$ of a sequence of logical variables $\mathcal{X} = (X^1, \dots, X^n)$ and a set $C_{\mathcal{X}} \subseteq \times_{i = 1}^n\mathcal{D}(X_i)$.
	The symbol $\top$ for $C$ marks that no restrictions apply, i.e., $C_{\mathcal{X}} = \times_{i = 1}^n\mathcal{D}(X_i)$.
	A \emph{PRV} $R(L^1, \dots, L^n), n \geq 0$ is a syntactical construct of a random variable $R \in \mathbf{R}$ possibly combined with logical variables $L^1, \dots, L^n \in \mathbf{L}$. 
	If $n = 0$, the PRV is parameterless and forms a propositional random variable.
	A PRV $A$ or logical variable $L$ under constraint $C$ is given by $A_{|C}$ or $L_{|C}$, respectively.
	We may omit $|\top$ in $A_{|\top}$ or $L_{|\top}$.
	The term $\mathcal{R}(A)$ denotes the possible values (range) of a PRV $A$.
	An \emph{event} $A = a$ denotes the occurrence of PRV $A$ with range value $a \in \mathcal{R}(A)$.
\end{definition}

We use the random variable names $L$, $M$, $N$, and $O$ as well as $\mathbf{L} = \{X, Y\}$ with $\mathcal{D}(X) = \{x_1, x_2, x_3\}$ and $\mathcal{D}(Y) = \{y_1, y_2\}$.
Using the names, we build boolean \acp{prv} $N$, $L(X)$, $O(Y)$, and $M(X,Y)$.
To set \acp{prv} into relation, we use \acp{pf}.
A \ac{pf} describes a function, mapping argument values to real values, of which at least one is non-zero.

\begin{definition}
	We denote a \emph{parfactor} $g$ by $\phi(\mathcal{A})_{| C}$ with $\mathcal{A} = (A^1, \dots, A^n)$ a sequence of PRVs, $\phi : \times_{i = 1}^n \mathcal{R}(A^i) \mapsto \mathbb{R}^+$ a function with name $\phi \in \Phi$, and $C$ a constraint on the logical variables of $\mathcal{A}$. 
	We may omit $|\top$ in $\phi(\mathcal{A})_{| \top}$.
	The term $lv(Y)$ refers to the \acp{lv} in some element $Y$, a PRV, a parfactor or sets thereof.
	The term $gr(Y_{| C})$ denotes the set of all instances of $Y$ w.r.t.\ constraint $C$.
	A set of parfactors forms a \emph{model} $G := \{g^i\}_{i=1}^n$.
	The semantic of $G$ is given by grounding and building a full joint distribution.
	With $Z$ as the normalisation constant, $G$ represents $P_G = \frac{1}{Z} \prod_{f \in gr(G)} f$.
\end{definition}

\cref{fig:swe} shows the graphical representation of a \ac{pm}. 


\subsection{Adding Time to the Representation}\label{sec:pdm}

Similar to a \ac{dbn}, there are two assumptions for \acp{pdm}, namely that the underlying process is stationary and, that the first-order Markov assumption holds.
Thus, we can define a \ac{pdm} analogously to an \ac{dbn} over an initial model and a temporal copy pattern.

\begin{definition}
	A \ac{pdm} $G$ is a pair of \acp{pm} $(G_0,G_\rightarrow)$ where
        $G_0$ is a PM for the first time step and
        $G_\rightarrow$ is a two-slice temporal \ac{pm} representing a \ac{pm} for time slice $t-1$ and the same \ac{pm} for time slice $t$.
		The time slices are connected by so-called \emph{inter}-slice \acp{pf} to model the temporal behaviour.
		So $G_0$ is the model for the initial time step, which then can be extended using the temporal copy pattern $G_\rightarrow$ for any number of time steps.
        The semantics of $G$ is to unroll $G$ for $T$ time steps resulting in a \ac{pm} as defined above.
\end{definition}

Assume the initial model is the \ac{pm} from \cref{fig:swe} with time step $0$ (and possibly priors added).
\Cref{fig:TSPG} illustrates the temporal copy pattern.
The temporal copy pattern consists of an \ac{pm} for time slice $t-1$ and one for time slice $t$  as well as an \emph{inter}-slice \ac{pf} $g_t^i$.

%
Having a \ac{pdm}, we can ask queries on the model.
    Given a \ac{pdm} $G$, a query term $Q$ (ground \ac{prv}), and events $\mathbf{E}_{0:t} = \{E^i_t=e^i_t\}_{i,t}$, the expression $P(Q_t|\mathbf{E}_{0:t})$ denotes a \emph{query} w.r.t.\ $P_G$.
The problem of answering a query $P(A^i_\pi|\mathbf{E}_{0:t})$ w.r.t.\ the model is called  \emph{filtering} for $\pi = t$, \emph{prediction} for $\pi > t$, and \emph{hindsight} for $\pi < t$.
In this paper, we analyse under which circumstances such queries are guaranteed to have a lifted solution as well as the complexity of such queries.

\subsection{Query Answering Algorithm: LDJT}\label{sec:fodjt}

\ac{ldjt} is an exact algorithm to efficiently answer multiple queries in \acp{pdm} \citep{gehrke2018ldjt}. 
\ac{ldjt} constructs a so-called \ac{fojt} \citep{BrMoe16a} from a \ac{pdm} to answer multiple queries using \ac{lve} operations.
An \ac{fojt} has parameterised clusters (parclusters) as nodes.
Similar to the propositional case, where the size of the clusters corresponds to the treewidth \citep{darwiche2009modeling}, we use parclusters in our complexity analysis for the lifted width.
Parclusters divide a model into submodels using conditional independencies.
\ac{ldjt} performs a so-called message passing to distribute local information to prepare parclusters for query answering.
Then, queries can be answered on these smaller parclusters.

Given the first-order Markov assumption of a \ac{pdm}, \ac{ldjt} can identify \acp{prv} in the temporal copy pattern, which makes one time step conditionally independent from the next.
We call the set of these \acp{prv} interface.
The interface $\mathbf{I}_{t-1}$ consists of all \acp{prv} from time step $t-1$, which occur in an \emph{inter}-slice \acp{pf}.
In \cref{fig:TSPG}, the highlighted \acp{prv} $M_{t-1}(X,Y)$ and $L_{t-1}(X)$ make up the interface $\mathbf{I}_{t-1}$.
The joint distribution of $\mathbf{I}_{t-1}$ makes time step $t-1$ and $t$ conditionally independent.
During the construction of \ac{ldjt}'s \ac{fojt} structures, it ensures that one parcluster contains at least the \acp{prv} of $\mathbf{I}_{t-1}$, which is called \emph{in-cluster}, and also constructs a parcluster containing $\mathbf{I}_{t}$, which is called \emph{out-cluster}.
To proceed in time, \ac{ldjt} computes a message over $\mathbf{I}_{t}$, using the \emph{out-cluster} of $J_t$.
Then, \ac{ldjt} sends the result, which we call $\alpha_t$ message, to the \emph{in-cluster} of $J_{t+1}$.

    \begin{figure*}[t]
    \center
    \begin{tikzpicture}[every node/.style={font=\footnotesize}, node distance=45mm]
    
    	\node[pc, label={[gray, inner sep=1pt]270:{$\{g^i_{t-1}\}$}},    pin={[pin distance=1mm, gray, align=center]70:{in-cluster}}, label={[font=]90:{$\mathbf{C}^1_{t-1}$}}]				(c1) {$M_{t-1}(X,Y),$ \\ $L_{t-1}(X), $\\ $M_{t-1}(X,Y)$};
    
    	\node[pc, right of=c1,     pin={[pin distance=1mm, gray, align=center]70:{\textbf{out-cluster}}},
    label={[gray, inner sep=1pt]270:{$\{g^0_{t-1}, g^1_{t-1}\}$}},label={90:{$\mathbf{C}^2_{t-1}$}}]	(c2) {$N_{t-1}, O_{t-1}(Y),$\\$L_{t-1}(X),$ \\ $M_{t-1}(X,Y)$};
    
        \node[below of = c2, node distance=1cm, xshift= 2.35cm] (c11) {$\alpha_{t-1}$};
        \node[below of = c2, node distance=1cm, xshift= 0.4cm] (a) {};
    
    	\node[pc, right of=c2, node distance=4.5cm, label={[gray, inner sep=1pt]270:{$\{g^i_t\}$}},  pin={[pin distance=1mm, gray, align=center]120:{\textbf{in-cluster}}}, label={[font=]90:{$\mathbf{C}^1_t$}}]				(c4) {$M{_{t-1}}(X,Y),$ \\ $L_{t-1}(X),$\\ $M_t(X,Y)$};
    	\node[pc, right of=c4,  node distance = 5cm,   pin={[pin distance=1mm, gray, align=center]120:{out-cluster}},
    label={[gray, inner sep=1pt]270:{$\{g^0_t, g^1_t \}$}},label={90:{$\mathbf{C}^2_t$}}]	(c5) {$N_t, O_t(Y),$\\$ L_t(X),$ \\ $M_t(X,Y)$};
        \node[below of = c4, node distance=1cm, xshift= -0.3cm] (b) {};

    	\draw (c1) -- node[inner sep=1pt, pin={[yshift=-2.4mm]90:{$\{M_{t-1}(X,Y)\}$}}]		{} (c2);

    	\draw (c4) -- node[inner sep=1pt, pin={[yshift=-2.4mm]90:{$\{M_t(X,Y)\}$}}]		{} (c5);

        \path [<-] (c11) edge node [] {} (a);
        \path [->] (c11) edge node [] {} (b);

    \end{tikzpicture}
    \caption{Groundings due to Temporal Elimination Order ($J_{t-1}$ on the left and $J_t$ on the right)}
    \label{fig:notComplete}	
    \end{figure*}

\Cref{fig:notComplete} depicts an \ac{fojt} $J_{t-1}$ for time step $t-1$ on the left and an \ac{fojt} $J_t$ for time step $t$ on the right as well as the $\alpha_{t-1}$ message to make these two \acp{fojt} conditionally independent.
\ac{ldjt} computes the $\alpha_{t-1}$ message over $M_{t-1}(X,Y)$ and $L_{t-1}(X)$ using the \emph{out-cluster} of $J_{t-1}$ and sends $\alpha_{t-1}$ to the \emph{in-cluster} of $J_t$. 
During all computations, such as computing a message, \ac{ldjt} tries to avoid groundings.
To obtain a so-called lifted solution, \ac{ldjt} aims to eliminate \acp{prv} using lifted summing out, which eliminates \acp{prv} using representatives and then efficiently accounts for ground eliminations.
Thereby, \ac{ldjt} achieves tractability through exchangeability \citep{niepert2014tractability}.


There are some preconditions to being able to perform computations in a lifted fashion.
Therefore, let us have a closer look at lifted summing out and counting as an enabler for lifted summing out.
Lifted summing out is applicable on a \ac{prv} $A$ from a \ac{pf} $g$ if $lv(A) \supseteq lv(g)$ holds.
That is to say a \ac{prv} $A$ from a \ac{pf} $g$ can only be eliminated if all \acp{lv} from $g$ also appear in $A$.
If lifted summing out is not directly applicable, there are operations such as counting as an enabler for lifted summing out.

To provide an idea of lifted summing out as well as counting, assume we only have $g_{t-1}^1$ and want to eliminate $O_{t-1}(Y)$.
To eliminate $O_{t-1}(Y)$ from $g_{t-1}^1$, \ac{ldjt} cannot apply lifted summing out, as $L_{t-1}(X)$ is over the logical variable $X$ ($\{Y\} \nsupseteq \{Y, X\}$).
Thus, $O_{t-1}(Y)$ is not over all logical variables of $g_{t-1}^1$.
In such a situation, \ac{ldjt} can still obtain a lifted solution by counting the logical variable $X$ in $L_{t-1}(X)$.
Instead of grounding, count-converting counts indistinguishable cases of the underlying random variables into a set of histograms to only have a polynomial instead of an exponential blowup.
By counting, \ac{ldjt} binds the counted logical variable and transforms \acp{prv}, which are parameterised with that logical variable, into a set of histograms using the encoded symmetries, i.e., for a boolean PRV with a domain of two one would have three histograms. 
The three histograms would be: $[2,0]$, all true, $[1,1]$, one true and one false, and $[0,2]$, all false and each histogram maps to one potential.
Hence, counting exploits indistinguishability to only have a polynomial blowup instead of an exponential.
In general, counting has two preconditions.
The first precondition is that the logical variable does not appear inside a counting formula in the parfactor.
The second precondition is that the logical variable does not appear inside a constraint associated with a counting formula.
So the logical variable is not allowed to occur in an inequality constraint with an already counted logical variable.
For the second precondition, Taghipour introduces a merge count operator to handle the second case.
As the count-conversion binds $X$, the $Y$ of $O_{t-1}(Y)$ is the only remaining logical variable in $g_{t-1}^1$.
Thus, \ac{ldjt} applies lifted summing out to eliminate $O_{t-1}(Y)$.
For a more detailed introduction, as well as pre- and post-conditions, for lifted summing out and count-conversion, please refer to \cite{Taghipour2013lifted}. We also provide the operators in \cref{app:lve}.


As mentioned, \ac{ldjt} has to compute messages over sets of \acp{prv} to make two parclusters conditionally independent from each other.
Preserving \acp{prv} for a message influences the elimination order, which in turn might lead to groundings.
Therefore, in the static case \ac{ljt} has a so-called \emph{fusion} step, which checks if groundings occur while calculating a message between parclusters \citep{braun2017preventing}.
In case the groundings do not occur in larger submodels, the \emph{fusion} step merges parclusters where calculating a message would otherwise result in unnecessary groundings.
Thereby, the \emph{fusion} step might enlarge parclusters but prevents algorithm-induced groundings.
In the worst case, \emph{fusion} merges all parclusters, resulting in applying \ac{lve} on the complete model.
In the temporal case, a so-called \emph{extension} step exists, which checks for unnecessary grounding while computing temporal messages and delays eliminations to prevent unnecessary groundings \citep{GehBrMo18c, GehBrMo18e}.
Delaying eliminations can be considered as enlarging the set of interface variables.
Here the worst case would be to have all \acp{prv} from $t-1$ in the interface given our first-order Markov assumption.
In general \emph{fusion} only enlarges submodels and \emph{extension} the interface if necessary.
However for the completeness analysis, we can also just assume the worst case, which is equal to computing the temporal message, over all PRVs from a given time step, by applying LVE on the temporal copy pattern.
By assuming the worst case, we can avoid having to understand \emph{fusion} and \emph{extension} in depth.
With \emph{fusion} and \emph{extension} in the correct places the analysis would also be harder to follow.
Thus, in the following, we sometimes assume the worst case, even if it would not occur in reality.
Nonetheless, the checks can also be found in \cref{app:prev}.

In the next section, we have a look at a pattern in the inter-slice \acp{pf} where \ac{ldjt} has to ground also for the worst case.

\section{Temporal Elimination Order and Lifted Solutions}
In the following, we look at a case where \ac{ldjt} cannot guarantee a lifted solution for a model from $FO^2$.
We use the insights of this case to show that the classes of liftable models have to be restricted for \ac{ldjt} w.r.t.\ static inference algorithms.
Next, we use our example from \cref{fig:notComplete} to illustrate when groundings occur.
Then, we derive a pattern, which leads to groundings.

%

\begin{example}[Groundings LDJT cannot prevent]\label{ex:notComplete}
    \Cref{fig:notComplete} depicts \acp{fojt} of $G_\rightarrow^{ex}$ for two time steps. 
    Let us start by trying to proceed in time as depicted in \cref{fig:notComplete} before we have a look at the model with the fewest restrictions, the worst case from the previous section.
    Recall our interface $I$ consists of $M_{t-1}(X,Y)$ and $L_{t-1}(X)$.
    To compute $\alpha_{t-1}$, \ac{ldjt} could multiply $g_{t-1}^0$ and $g_{t-1}^1$.
	Then it could count-convert $X$ in the resulting parfactor.
	Now, $Y$ is the only logical variable remaining. 
	Thus, LDJT can eliminate $O_{t-1}(Y)$ using lifted summing out.
	Next, LDJT would try to eliminate $N$ in a lifted manner.
	Unfortunately, $M_{t-1}(X,Y)$ still has the free logical variable $Y$.
	$Y$ cannot be count-converted as it appears in a counting formula.
	Thus, LDJT would have to ground.
	LDJT also could not have counted $X$ and $Y$ initially as the counting operator is only defined for one logical variable.
    Overall, here \ac{ldjt} cannot proceed in time without grounding.
    
    Let us now have a look at the worst case from the previous section, i.e., use the maximal interface, i.e., $M_{t-1}(X, Y)$, $L_{t-1}(X)$, $O_{t-1}(Y)$ and $N_{t-1}$ as well as do not cluster the model into submodels.
    We can see that actually in our temporal copy pattern in \cref{fig:TSPG}.
    Given that model, \ac{ldjt} has to compute a query $P(M_{t}(X,Y), L_{t}(X), O_{t}(Y), N_{t})$, which is our $\alpha_{t}$.
    The idea here is to delay the eliminations of $N_{t-1}$ and $O_{t-1}(Y)$ to try to prevent the groundings.
	In case LDJT can eliminate all PRVs from $t-1$ without having to ground, it is able to compute $\alpha_{t}$ in a lifted fashion.
    However, as we will see we run into the identical problem as before.
    \ac{ldjt} can eliminate $M_{t-1}(X,Y)$ by multiplying $g^0_{t-1}$ and $g^i$ and then applying lifted summing out.
    To eliminate either of $N_{t-1}$, $O_{t-1}(Y)$, or $L_{t-1}(X)$ \ac{ldjt} first has to multiply the previous result with $g^1_{t-1}$
    Afterwards, \ac{ldjt} has a \ac{pf} that includes $N_{t-1}$, $O_{t-1}(Y)$, $L_{t-1}(X)$, and $M_t(X,Y)$.
    From that \ac{pf}, \ac{ldjt} needs to eliminate $N_{t-1}$, $O_{t-1}(Y)$, and $L_{t-1}(X)$.
	Lifted summing out can only be applied $M_t(X, Y)$ in this case as it is the only PRV with all the logical variables from the constraint of the parfactor.
	Further, counting is only defined for one logical variable at a time.
	So LDJT cannot count the $X$ and the $Y$ of $M_t(X, Y)$.
	Nonetheless, LDJT can count either $X$ or $Y$.
	In case $X$ is counted, LDJT eliminates $O_{t-1}(Y)$ using lifted summing out.
	In case $Y$ is counted, LDJT eliminates $L_{t-1}(X)$ using lifted summing out.
	However, in both cases, the remaining two (P)RVs cannot be eliminated in a lifted fashion, as in both cases $M_t(X, Y)$ still has one free logical variable.
	Therefore, to proceed \ac{ldjt} has to ground $M_t(X, Y)$ before it can proceed with eliminating the remaining (P)RVs from $t-1$.
    Thus, while proceeding in time using the temporal conditional independences with a first-order Markov assumption, \ac{ldjt} has to ground for this particular model.
	
\end{example}

The problem in \cref{ex:notComplete} is that the \ac{prv} $M(X, Y)$ occurs for $t$ and $t+1$ in the inter-slice \acp{pf}.
To compute the $\alpha_{t-1}$ message, \ac{ldjt} has to preserve $M_{t-1}(X,Y)$.
However, to eliminate other \acp{prv} without groundings, \ac{ldjt} has to eliminate $M_{t-1}(X, Y)$ first.
Delaying the elimination of the other \acp{prv} to $J_t$ does not help either.
Here, \ac{ldjt} can eliminate $M_{t-1}(X,Y)$, but $M_{t}(X,Y)$ comes into play.
Again, \ac{ldjt} would have to eliminate $M_{t}(X, Y)$ before it can eliminate the remaining \acp{prv} from time step $t-1$.
However, $M_{t}(X, Y)$ cannot be eliminated as \ac{ldjt} needs it to compute $\alpha_{t}$.

So the general problem here is that in the inter-slice \acp{pf} there is a \ac{prv} with two logical variables for time step $t$ as well as for time step $t+1$.
Such a case can result in groundings, but it does not necessarily has to lead to groundings. 
In the case that such a pattern leads to groundings, a lifted solution is possible by ignoring the temporal aspects.
Here, one can provide the unrolled \ac{pdm} to \ac{ljt}.
That would result in a parcluster, which contains the \ac{prv} $M(X, Y)$ for all $T$ time steps.
Having all \acp{prv} in a single parcluster results in performing \ac{lve} on the complete model for each query, which is inefficient for multiple queries.
Nonetheless, by ignoring the temporal aspects of the model, a lifted solution is possible.
Here, \ac{ljt} starts by eliminating $M(X, Y)$ for all time steps and then proceeds with other \acp{prv}.
Depending on domain sizes and the maximum number of time steps, either using \ac{ldjt} with groundings or using the unrolled model with \ac{ljt} is advantageous.
The trade-off has been empirically evaluated (for one example please see \cite{GehBrMo18e}).

Knowing that temporal conditional independences induce restrictions on the elimination order, we now have a look at the completeness of \ac{ldjt} and compare it to known results for static algorithms.

\section{Completeness}

\citeauthor{Bra20} shows that \ac{ljt} is complete for models from $FO^2$ \cite{Bra20} and \citeauthor{taghipour2013completeness} show that \ac{lve} with generalised counting is complete for models from $FO^2$ \cite{taghipour2013completeness}.
The same completeness results also hold for other exact static lifted inference algorithms \citep{broeck2011completeness}.
Models from $FO^2$ have at most two \acp{lv} in each \ac{pf}.


\begin{corollary}\label{thm:complete}
    \ac{lve} and \ac{ljt} are complete for any \ac{pdm} $G$ from $FO^2$.
\end{corollary}

\begin{proof}
Unrolling a temporal model results in a static model.
If a \ac{pdm} $G$ is from $FO^2$, then the unrolled version of $G$ is in $FO^2$ and \ac{lve} as well as \ac{ljt} are complete for models from $FO^2$.
\end{proof}

Hence, by not accounting for temporal conditional independences, \ac{lve} and \ac{ljt} answer queries in polynomial time w.r.t.\ the domain size for any \acp{pdm} from $FO^2$.
Let us now include temporal conditional independences for the completeness analysis and have a look at \ac{ldjt}. 
For our completeness analysis of \ac{ldjt}, we begin with a negative result, which directly follows from \cref{ex:notComplete}.
Afterwards, we specify the model classes for which \ac{ldjt} is complete and therefore can guarantee query answering results in a reasonable time, leading to the knowledge of tractable inference for (temporal) model classes.

\begin{theorem}\label{thm:complete2}
    \ac{ldjt} is not complete for $FO^2$.
\end{theorem}

\begin{proof}\label{pf:c2}
    \Cref{ex:notComplete} shows a model from $FO^2$ and \ac{ldjt} cannot compute a lifted solution for that model.
    Hence, \ac{ldjt} is not complete for all $FO^2$ models.
\end{proof}

So where is the difference between \ac{lve} and \ac{ldjt}?
By unrolling $G^{ex}$ and using \ac{lve}, \ac{lve} can start by eliminating all occurrences of the \ac{prv} $M_t(X, Y)$.
In the static case, there is no restriction w.r.t.\ the elimination order and thus \ac{lve} can start by eliminating all PRVs with two logical variables, then proceed with PRVs with one logical variable and finally eliminate random variables.
However, with \ac{ldjt}, we aim at handling temporal aspects efficiently, which is not given anymore by performing \ac{lve} on the unrolled model as temporal independences are not accounted for.
For a large $T$, the unrolled model will be enormous. 
In our toy example, we have four PRVs.
Assuming $1,000,000$ time steps, the unrolled model would have to store $4,000,000$ PRVs at once in memory instead of $8$ at most.
Additionally, if one would apply LJT on the unrolled model, the heuristic to construct its \ac{fojt}, would most likely cluster all occurrences of $M_t(X, Y)$ together. 
Hence, it would have a single cluster with at least $1,000,000$ PRVs and the lifted width is dependent on the largest cluster.
A large lifted width then also directly has a great influence on the corresponding runtime to perform inference.
Therefore, \ac{ldjt} trades in completeness to handle temporal aspects efficiently.

Maybe a good comparison of why accounting for temporal conditional independences is hard and influences the completeness results is that the problem is similar to maximum a posteriori (MAP) queries.
In MAP queries, first some \acp{prv} have to be summed out and then the remaining ones have to be maxed out. 
The problem is similar as both influence the elimination order.
The PRVs that have to be maxed out have to be preserved.
Thus, these queries also introduce restrictions w.r.t.\ possible elimination orders.
\ac{ldjt} has to preserve a set of \acp{prv}, similar to the \acp{prv} that have to be maxed-out in a MAP query and MAP queries are also not complete for $FO^2$ \citep{Bra20}.

\begin{conjecture}
    Any lifted algorithm using temporal conditional independences is not complete for $FO^2$.
\end{conjecture}

Similar to MAP queries, by accounting for temporal conditional independences some \acp{prv} have to be preserved or in other words, some \acp{prv} have to be eliminated before others.
Thus, possible elimination orders are restricted for all lifted algorithms using temporal conditional independences in the same fashion.
Hence, all lifted algorithms using temporal conditional independences have to deal with the very same problem.

Nonetheless, for the completeness of \ac{ldjt}, only one distinct pattern in the \emph{inter-slice} \acp{pf} has to be excluded, i.e., one pattern while moving from one time step to the next.
Now, we prove for which model classes \ac{ldjt} is complete, leading to new theoretical bounds.
Let us call the class of models from $FO^2$ with inter-slice \acp{pf} in which \acp{prv} from at most one time slice are parameterised with at most two logical variables $TFO^2$.
In the proof, we focus solely on the computations between time steps as the computations within time steps are complete.


\begin{theorem}\label{thm:complete3}
    \ac{ldjt} is complete for models from $TFO^2$.
\end{theorem}

\begin{proof}\label{pf:c3}
    For the proof, we consider the three remaining cases, namely:
    \renewcommand{\labelenumi}{\roman{enumi})}
    \begin{enumerate}
        \item only \acp{prv} with at most one \ac{lv} in inter-slice \acp{pf},
        \item only \acp{prv} with two \acp{lv} for time slice $t-1$ in inter-slice \acp{pf}, and
        \item only \acp{prv} with two \acp{lv} for time slice $t$ in inter-slice \acp{pf}.
    \end{enumerate}

    Case i) is similar to the proof that LVE is complete for $FO^2$. 
	To proceed in time, \ac{ldjt} does not have to preserve any \acp{prv} with two \acp{lv}.
	Thus, \ac{ldjt} can start by eliminating all \acp{prv} with two logical variables from time step $t-1$.
	Afterwards, there are only \acp{prv} with at most one logical variable from time step $t-1$.
	Now, \ac{ldjt} can eliminate the remaining \acp{prv} from time step $t-1$ that are not in the interface using the operators count-conversion, merge-count, and merge.
	Entering the temporal message in the next time step also does not lead to groundings as with multiply, merge-count, and merge the interface message(s) can be multiplied with the inter slice parfactors and then all PRVs from time step $t-1$ can be eliminated, as there are no PRVs with two \acp{lv} in the inter-slice \acp{pf}.
	Thus, \ac{ldjt} does not have to ground while proceeding in time and is complete for case i).

    Case ii) means that at least one \ac{prv} from $t-1$ with two \acp{lv} is in the interface.
	In \Cref{ex:notComplete} we have already seen that \ac{ldjt} might have to ground while computing the temporal message if it has to preserve a \ac{prv} with two \acp{lv}.
	For the proof, we consider the maximal interface, i.e., all PRVs from $t-1$ in the interface\footnote{\ac{ldjt} has means to identify the interface it needs to not have to ground. One does not want to use an unnecessarily large interface. We only do it here as it makes the proof easier to follow.}.
	Therefore, \ac{ldjt} does not have to eliminate any PRVs from $t-1$ to compute $\alpha_{t-1}$.
	In time step $t$, $\alpha_{t-1}$ is initially multiplied with the inter-slice \acp{pf}.
	Here, \ac{ldjt} can start by eliminating all \acp{prv} with two \acp{lv} for time slice $t-1$ using lifted summing out (as we only consider $FO^2$ the preconditions of lifting summing out have to hold here). 
	Given our case, we know that there are no PRVs with two \acp{lv} for $t$ in the inter-slice \acp{pf}.
    Thus, \ac{ldjt} can eliminate the remaining \acp{prv} from time-slice $t-1$ with generalised counting.
	Now, there are only PRVs from $t$ remaining.
	These PRVs from $t$ also make up $\alpha_t$.
	Thus, LDJT could eliminate all PRVs from $t-1$ without grounding and thereby compute $\alpha_t$ without having to ground. 
	LDJT can use those to answer queries and here \ac{ldjt} performs computations within a time step, which is proven to be complete for $FO^2$.
	Additionally, it can pass the remaining parfactors as the temporal message $\alpha_t$ to $t+1$, which will not result in any grounding and therefore, \ac{ldjt} solves this case in at most polynomial and not exponential time w.r.t.\ domain sizes.
    Hence, \ac{ldjt} is complete for case ii).

    Case iii) is similar to case i).
    To compute temporal messages $\alpha_{t-1}$ there are no algorithm-induced groundings as \ac{ldjt} does not have to preserve a PRV with two logical variables.
	Thus, there are no PRVs with two logical variables in $\alpha_{t-1}$.
	In $t$, \ac{ldjt} again starts by multiplying $\alpha_{t-1}$ with the inter-slice \acp{pf}.
	Here, \ac{ldjt} cannot directly eliminate all \acp{prv} from $t-1$ as there is at least one PRV with two logical variables from $t$ in the inter-slice \acp{pf}, let us call this PRV $A$.
	However, $A$ has not to be preserved to compute the next temporal message $\alpha_{t}$.
	Thus, to compute $\alpha_{t}$, LDJT can first eliminate $A$ using lifted summing out.
	Afterwards, LDJT can eliminate all PRVs not in the interface, which includes all PRVs from $t-1$, using generalised counting as each of them has at most one logical variable.
	Hence, LDJT also computes the temporal messages in the third case without having to ground and is also complete for case iii).
	
	So far, we only considered the forward pass to proceed in time. 
	For a backward pass, to answer hindsight queries, the argumentation for case ii) and case iii) are exchanged.
	Otherwise, nothing changes if we consider backward instead of forward messages.
    Therefore, overall \ac{ldjt} is complete for all three cases, which means that \ac{ldjt} is complete for models from $TFO^2$.
	
\end{proof}

Having analysed the completeness for \ac{ldjt} and $TFO^2$, we now have a look at another interesting model class, which has at most $1$ logical variable in each PRV ($\mathcal{M}^{1prv}$).

\begin{corollary}
    \ac{ldjt} is complete for $\mathcal{M}^{1prv}$.
\end{corollary}

\begin{proof}\label{pf:}
    The proof directly follows from \cref{thm:complete3} and \cite[Thm. 7.2]{Taghipour2013lifted}.
	Here, generalised counting ensures that \ac{ldjt} is complete for $\mathcal{M}^{1prv}$.
	The rough idea of the proof is that all PRVs can be counted.
	Counting only leads to a polynomial and not an exponential blowup.
	After counting no free variables remain and from the polynomial representation they can be directly eliminated\footnote{Again this is to show that it is polynomial w.r.t.\ domain sizes in the worst case. In reality, LDJT would eliminate as many PRVs as possible with lifted summing out as possible instead of counting them all.}.
	With at most 1 \ac{lv} for each \ac{prv} the problematic temporal pattern (a PRV with two logical variables for time step $t-1$ and time step $t$) cannot occur.
	Thus, \ac{ldjt} is complete for $\mathcal{M}^{1prv}$.
\end{proof}

In general, completeness results for relational inference algorithms assume liftable evidence.
In case evidence breaks symmetries, query answering might not run in polynomial time but in exponential time w.r.t.\ domain sizes.
Especially, for temporal models small differences in observations slowly ground a model.
However, for conditioning with liftable evidence, the problem of evidence breaking symmetries over time has been dealt with by approximating symmetries \citep{GehMoBr20}.
Further, even if an algorithm is not complete for a certain class, the algorithm might still compute a lifted solution for some models of that class.
Similar to static algorithms \citep{taghipour2013first}, \ac{ldjt} can calculate a lifted solution for a 3-\ac{lv} model, even though \ac{ldjt} is not complete for 3-\ac{lv} models.
The same holds for models from $FO^2$ with inter-slice \acp{pf} in which \acp{prv} from both time slices are parameterised with two logical variables.
Here, \ac{ldjt} can compute for some models a lifted solution. 


Knowing the boundaries of completeness and \ac{ldjt}, we now take a look at the complexity of \ac{ldjt}.

\section{Complexity}\label{sec:comp:ldjt}


For the complexity analysis, we have a look at the complexity of each step of \ac{ldjt}, before we compile the overall complexity of \ac{ldjt}.
The complete algorithm of LDJT is also given in \cref{app:ldjt}.
Each of these steps can also be found in the complexity analysis for \ac{ljt} \citep{Bra20} and the main results for LVE and LJT can also be found in \cref{app:comp}.
Therefore, we compare the complexity of \ac{ldjt} to \ac{ljt} and only analyse the differences.

\subsection{LDJT Complexity}

Let us now analyse the complexity of \ac{ldjt}.
\Cref{app:ldjt} provides the complete algorithm.
For \ac{ldjt}, the complexity of the \ac{fojt} construction is negligible.
Compared to \ac{ljt} \citep{Bra20}, the complexity of the \ac{fojt} construction of \ac{ldjt} only differs in constant factors.
The construction of the \ac{fojt} structures $J_0$ and $J_t$ is the very same as for \ac{ljt}.
The difference is that two \acp{fojt} are constructed.
For \emph{extension}, \ac{ldjt} additionally performs checks on two messages and twice the \emph{fusion} step of \ac{ljt}.
The additional checks are constant factors and therefore, do not change the complexity of the \ac{fojt} construction, which we can neglect \citep{Bra20}.

For the complexity analysis of \ac{ldjt}, we assume that the \ac{fojt} structures of \ac{ldjt} are minimal, i.e., they cease to be an \ac{fojt} if anything was removed \citep{BrMoe16a}, and do not induce groundings.
Further, we slightly change the definition of \emph{lifted width} \citep{taghipour2013completeness, Bra20}, as we now consider a \ac{pdm} $G$ and two \acp{fojt}, $J_0$ and $J_t$.

\begin{definition}
    Let $w_{J_0} = (w_g^0, w_\#^0)$ be the \emph{lifted width} of $J_0$ and let $w_{J_t} = (w_g^t, w_\#^t)$ be the \emph{lifted width} of $J_t$.
	The \emph{lifted width} $w_J$ of a pair $(J_0, J_t)$ is a pair $(w_g, w_\#)$, where $w_g = \max(w_g^0,w_g^t)$ and $w_\# = \max(w_\#^0,w_\#^t)$.
\end{definition}

Further, $T$ is the maximum number of time steps, $k$ the largest lag\footnote{Normally, the term lag is used for time differences in hindsight queries. We also use it for prediction queries.}, $n$ is the largest domain size among $lv(G)$, $n_{\scriptscriptstyle\#}$ is the largest domain size of the counted logical variables, $r$ is the largest range size in a $G$, $r_{\scriptscriptstyle\#}$ is the largest range size among the PRVs in the counted RVs, and $n_J = max(n_{j_0},n_{j_t})$ being the number of parclusters. 
The largest possible factor is given by $r^{w_g} \cdot n_{\scriptscriptstyle\#}^{w_\# \cdot r_\#}$.
Hence, we always look at the highest number that occurs either in $J_0$ or $J_t$.

\emph{Evidence entering} consists of absorbing evidence at each applicable node.
\begin{lemma}\label{prop:complex:dev}
	The complexity of absorbing an evidence parfactor is
	\begin{align}
		O(T \cdot n_J \cdot \log_2 (n) \cdot r^{w_g} \cdot n_{\scriptscriptstyle\#}^{w_\# \cdot r_\#}). \label{eq:complex:dev}
	\end{align}
\end{lemma}
Overall, there is evidence for up to $T$ time steps.
Therefore, \ac{ldjt} enters evidence in $T$ \acp{fojt}. 

\emph{Passing messages} consists of calculating messages with \ac{ljt} for every time step.
Here, we consider the worst case, i.e., for each time step querying the first and last time step, the average case, i.e., \emph{hindsight} and \emph{prediction} queries with a constant lag, and the best case, i.e., only \emph{filtering} queries.
\begin{lemma}\label{prop:complex:dmsg}
	The worst case complexity of passing messages is
	\begin{align}
		O(T^2 \cdot n_J \cdot \log_2 (n) \cdot r^{w_g} \cdot n_{\scriptscriptstyle\#}^{w_\# \cdot r_\#}). \label{eq:complex:dmsg:worst}
	\end{align}
	The average case complexity of passing messages is
	\begin{align}
		O(k \cdot T \cdot n_J \cdot \log_2 (n) \cdot r^{w_g} \cdot n_{\scriptscriptstyle\#}^{w_\# \cdot r_\#}). \label{eq:complex:dmsg:average}
	\end{align}
	The best case complexity of passing messages is
	\begin{align}
		O(T \cdot n_J \cdot \log_2 (n) \cdot r^{w_g} \cdot n_{\scriptscriptstyle\#}^{w_\# \cdot r_\#}). \label{eq:complex:dmsg:best}
	\end{align}
\end{lemma}

The complexity of one complete message pass in an \ac{fojt}, consists of calculating $2 \cdot (n_J -1)$ messages and each message has a complexity of $O(\log_2 n \cdot r^{w_g} \cdot n_{\scriptscriptstyle\#}^{w_\# \cdot r_\#})$ \citep{Bra20}.
One difference in \ac{ldjt} compared to \ac{ljt} is that \ac{ldjt} needs to calculate $2 \cdot (n_J -1) + 2$ messages for the current \ac{fojt}, because \ac{ldjt} calculates an $\alpha$ and a $\beta$ message in addition to the normal message pass.
For the \ac{fojt} used to answer \emph{prediction}\footnote{Please note that LDJT is defined to answer prediction queries by propagating the current state of the world from one time step to the next by probabilistic inference until the final destination is reached. Especially for prediction, \ac{ldjt} could compute the temporal behaviour of our stationary process once and then project the current state directly into the future, which is way more efficient, as for example proposed by Marwitz et al. \cite{MaMoGe23}. However, splitting the complexity into a part for hindsight, for which LDJT has to perform probabilistic inference for every time step and prediction, where the current state can be projected more efficiently into the future would make the complexity analysis drastically harder to follow. Therefore, we stick to the probabilistic inference solution for every time step, both for hindsight as well as prediction queries for this analysis.} or \emph{hindsight} queries, \ac{ldjt} calculates $2 \cdot (n_J -1) + 1$ messages, as \ac{ldjt} calculates either an $\alpha$ or $\beta$ message respectively. 
Additionally, \ac{ldjt} computes at least one message pass for each time step and at most a message pass for all time steps for each time step.
Therefore, we investigate the worst, average, and best case complexity of message passing. 

The worst case for \ac{ldjt} is that for each time step, there is a query for the first and the last time step.
Therefore, for each of the $T$ time steps, \ac{ldjt} would need to perform a message pass in all $T$ \acp{fojt}, leading to $T \cdot T$ message passes.
Hence, \ac{ldjt} would perform a message pass for the current time step $t$, a backward pass from $t$ to the first time step, which includes a message pass on each \ac{fojt} on the path, and a forward pass from $t$ to the last time step, which includes a message pass on each \ac{fojt} on the path.
These message passes are then executed for each time step.
Thus, \ac{ldjt} performs overall $T \cdot T$ message passes.
The complexity of \cref{eq:complex:dmsg:worst} is also the complexity of \ac{ljt} given an unrolled \ac{fojt} constructed by \ac{ldjt} and evidence for each time step.
However, usually one is hardly ever interested in always querying the first and the last time step.

The best case for \ac{ldjt} is that it only needs to answer \emph{filtering} queries. 
That is to say, it needs to calculate $2 \cdot (n_J -1) + 1$ messages for each \ac{fojt}, as \ac{ldjt} calculates an $\alpha$ for each \ac{fojt}.
Further, \ac{ldjt} needs to perform exactly one message pass on each instantiated \ac{fojt}.
Therefore, \ac{ldjt} needs to pass messages on $T$ \acp{fojt}.

The average case for \ac{ldjt} is that for each time step \ac{ldjt} answers a constant number of \emph{hindsight} and \emph{prediction} queries. 
Assume one defines a set of queries that is answered for each time step.
Such a set could for example include queries for $t-10$ and $t+15$.
From this set, \ac{ldjt} can identify the maximum lag $k$, which is in this case $15$.
Overall in this example, \ac{ldjt} needs to perform $25$ message passes to answer all queries for one time step.
Therefore, as an upper bound \ac{ldjt} passes messages $2 \cdot k \cdot T$ times.
In general, \emph{prediction} and \emph{hindsight} queries are often close to the current time step and $T$ can be huge ($k << T$).
Hence, we also have a look at the average case for \ac{ldjt}.

Under the presence of \emph{prediction} and \emph{hindsight} queries, \ac{ldjt} does not always need to calculate $2 \cdot (n_J -1)$ messages for each \ac{fojt}.
In case \ac{ldjt} has no query for time step $t$, but only needs $J_t$ to calculate an $\alpha$ or $\beta$ message, then calculating $(n_J -1)$ messages suffice for $J_t$.
By selecting the \emph{out-cluster} for a \emph{prediction} queries and respectively the \emph{in-cluster} for a \emph{hindsight} queries as root, then all required messages are present to calculate an $\alpha$ or $\beta$ message.
Hence, an efficient query answering plan reduces the complexity of message passing with constant factors.

The last step is \emph{query answering}, which consists of finding a parcluster and answering a query on an assembled submodel.
For query answering, we combine all queries for all time steps in one set. 

\begin{lemma}\label{prop:complex:dqa}
	The complexity of answering a set of queries  $\{Q_k\}_{k=1}^m$ is
	\begin{align}
		O(m \cdot \log_2 (n) \cdot r^{w_g} \cdot n_{\scriptscriptstyle\#}^{w_\# \cdot r_\#}). \label{eq:complex:dqa}
	\end{align}
\end{lemma}
The complexity for query answering in \ac{ldjt} does not differ from the complexity of \ac{ljt}. 

We now combine the stepwise complexities to arrive at the complexity of \ac{ldjt} by adding up the complexities in \cref{eq:complex:dev,eq:complex:dmsg:worst,eq:complex:dmsg:average,eq:complex:dmsg:best,eq:complex:dqa}.
\begin{theorem}\label{prop:complex:dljt}
	The worst case complexity of \ac{ldjt} is
	\begin{align}
		O(((T^2 + T ) \cdot  n_J + m) \cdot \log_2 (n) \cdot r^{w_g} \cdot n_{\scriptscriptstyle\#}^{w_\# \cdot r_\#}). \label{eq:complex:ldjt:worst}
	\end{align}
	The average case complexity of \ac{ldjt} is
	\begin{align}
		O((k \cdot T \cdot  n_J + m) \cdot \log_2 (n) \cdot r^{w_g} \cdot n_{\scriptscriptstyle\#}^{w_\# \cdot r_\#}). \label{eq:complex:ldjt:average}
	\end{align}
	The best case complexity of \ac{ldjt} is
	\begin{align}
		O((T  \cdot  n_J + m) \cdot \log_2 (n) \cdot r^{w_g} \cdot n_{\scriptscriptstyle\#}^{w_\# \cdot r_\#}). \label{eq:complex:ldjt:best}
	\end{align}
\end{theorem}

%

\subsection{LJT Comparison}\label{theo:ljt}

As already mentioned, in case \ac{ljt} uses the very same heuristic to construct its \ac{fojt} as \ac{ldjt} does, i.e., unroll \ac{ldjt}'s structure for $T$ time steps, \ac{ljt} has the worst case complexity of \ac{ldjt}.
The message pass of \ac{ljt} includes for every time step the first and the last time step.
Therefore, the worst case complexity of \ac{ldjt} is the best \ac{ljt} can achieve.
Some additional effort w.r.t.\ caching and efficient entering of adaptive evidence \citep{BraMo18e} have to be implemented to be nearly as efficient as \ac{ldjt} is in its worst case.
Further, for \ac{ldjt} for \emph{hindsight} and \emph{prediction} queries, computing $(n_J -1)$ messages instead of $2 \cdot (n_J -1)$ messages suffices.
A highly adapted and efficient implementation of \ac{ljt} could achieve that reduction at best only for \emph{prediction} queries, but cannot achieve that for \emph{hindsight} queries.
Thus, a lot of work would have to be spent to make \ac{ljt} as efficient as \ac{ldjt} in the worst case.
Therefore, in this case, \ac{ldjt} is always more efficient compared to \ac{ljt}.

The other possibility is that \ac{ljt} uses any static heuristic to build its \ac{fojt}.
In this case, \ac{ljt} would most likely have fewer, but larger parcluster.
The number of parclusters is only a linear factor for \ac{ljt}.
However, the size of the parclusters directly influences the lifted width.
Increasing the lifted width results in increasing an exponential term.
Thus, the complexity of \ac{ljt}, using a random (static) heuristic, will most likely be higher than the worst case complexity of \ac{ldjt}. 

\subsection{Interface Algorithm Comparison}\label{theo:ia}

For \ac{ljt} compared to a tree junction tree algorithm, the speed up is twofold.
The first speed-up is that \ac{ljt} has fewer nodes in an \ac{fojt} than the corresponding tree jtree has, i.e., $n_{gr(J)} >> n_J$.
The other speed up originates from the counted part, $n_{\scriptscriptstyle\#}^{w_\# \cdot r_\#}$.
Under the presence of counting, the lifted width is smaller than the treewidth.
Further, \citeauthor{taghipour2013first} shows that in models that do not require count-conversions, the lifted width is equal to the treewidth \cite{taghipour2013first}.
Therefore, the factor $r^{w_g}$ of the complexity of \ac{ljt} is the same as the treewidth of a junction tree algorithm without count-conversions.

\ac{ldjt} has another advantage over the interface algorithm.
Namely, for lifted temporal inference, even without counting, the lifted width is often much smaller than the treewidth.

\begin{theorem}
	The lifted width is always smaller than the treewidth given that one domain size in the set of interface \acp{prv} is larger than one.
\end{theorem}

\begin{proof}
	The interface algorithm ensures that all interface random variables are grouped in one cluster of a jtree.
	Therefore, the corresponding jtree has at least $|gr(\mathbf{I}_t)|$ random variables in a cluster.
	Thus, the treewidth of the interface algorithm depends on the domain sizes of the interface \acp{prv}.
	The $w_g$ (which corresponds to the treewidth) of the lifted width $w_J$ is independent of the domain sizes.
	Hence, the lifted width, even without count-conversions, is always smaller than the treewidth if one domain size in the set of interface \acp{prv} is larger than one.
	In the case of count-conversions, the lifted width is anyhow always smaller than the ground width \citep{taghipour2013first}.
\end{proof}

For \ac{lve} and \ac{ljt} without counting the lifted width and the treewidth are identical \citep{taghipour2013first}.
For \ac{ldjt} even without counting, the lifted width can be drastically smaller than the treewidth and propositional inference is exponential w.r.t.\ the treewidth \citep{darwiche2009modeling}.
For \ac{ldjt}, increasing domain sizes (without counting) only leads to increasing a logarithmic factor.
Thus, \ac{ldjt} can answer queries for large(r) domain sizes, which is infeasible in the propositional case due to an exponential blowup.
This comparison implies that without lifting even for small domain sizes the problem becomes infeasible.
Thus, lifting is necessary for temporal inference.


\section{Conclusion}

In this paper, we present completeness and complexity results for \ac{ldjt}.
To the best of our knowledge, this is the first such analysis for a temporal lifted inference algorithm.
\ac{ldjt} uses temporal conditional independences, which induces restrictions on the elimination order. 
Based on these restrictions, we show that an adjustment in necessary and prove a sufficient adjustment to the classes of liftable models w.r.t.\ completeness for static lifted inference algorithms.
We only have to exclude one special case in the inter-slice parfactors.
Thereby, we show that \ac{ldjt} can guarantee a lifted solution for many models and thereby compute solutions in a reasonable time.
For nearly all (realistic) scenarios, \ac{ldjt} has a complexity linear to the maximum number of time steps, which is the desired behaviour for an exact temporal inference algorithm.
Additionally, we also show how crucial lifting is for exact temporal inference.
Here, even without counting, the lifted width is much smaller compared to the treewidth.
For static algorithms, without counting, the lifted width is equal to the treewidth.
With recent advances in preserving lifted models over time \citep{GehMoBr20}, \ac{ldjt} preserves the advantages of lifting while answering conditioning queries, i.e., incorporating evidence.

Knowing the benefits of a lifted temporal inference algorithm, the next steps contain investigating how other sequential problems such as decision making can benefit from lifting.
Further, we want to find an approach to generalise counting even further to make \ac{ldjt} complete for $FO^2$.


\bibliography{../../bib/tex/bib}

\begin{thebibliography}{32}
\providecommand{\natexlab}[1]{#1}
\providecommand{\url}[1]{\texttt{#1}}
\expandafter\ifx\csname urlstyle\endcsname\relax
  \providecommand{\doi}[1]{doi: #1}\else
  \providecommand{\doi}{doi: \begingroup \urlstyle{rm}\Url}\fi

\bibitem[Ahmadi et~al.(2013)Ahmadi, Kersting, Mladenov, and
  Natarajan]{ahmadi2013exploiting}
B.~Ahmadi, K.~Kersting, M.~Mladenov, and S.~Natarajan.
\newblock {Exploiting Symmetries for Scaling Loopy Belief Propagation and
  Relational Training}.
\newblock \emph{Machine learning}, 92\penalty0 (1):\penalty0 91--132, 2013.

\bibitem[Braun(2020)]{Bra20}
T.~Braun.
\newblock \emph{{Rescued from a Sea of Queries: Exact Inference in
  Probabilistic Relational Models}}.
\newblock PhD thesis, 2020.

\bibitem[Braun and M{\"o}ller(2016)]{BrMoe16a}
T.~Braun and R.~M{\"o}ller.
\newblock {Lifted Junction Tree Algorithm}.
\newblock In \emph{Proceedings of {KI} 2016: Advances in Artificial
  Intelligence}, pages 30--42. Springer, 2016.

\bibitem[Braun and M{\"o}ller(2017)]{braun2017preventing}
T.~Braun and R.~M{\"o}ller.
\newblock {Preventing Groundings and Handling Evidence in the Lifted Junction
  Tree Algorithm}.
\newblock In \emph{Proceedings of {KI} 2017: Advances in Artificial
  Intelligence}, pages 85--98. Springer, 2017.

\bibitem[Braun and M\"oller(2018)]{BraMo18e}
T.~Braun and R.~M\"oller.
\newblock {Adaptive Inference on Probabilistic Relational Models}.
\newblock In \emph{Proceedings of the 31st Australasian Joint Conference on
  Artificial Intelligence}. Springer, 2018.

\bibitem[Braun and M{\"o}ller(2018)]{braun2018parameterised}
T.~Braun and R.~M{\"o}ller.
\newblock {Parameterised Queries and Lifted Query Answering}.
\newblock In \emph{IJCAI-18 Proceedings of the 27th International Joint
  Conference on Artificial Intelligence}, pages 4980--4986. International Joint
  Conferences on Artificial Intelligence Organization, 2018.

\bibitem[Darwiche(2009)]{darwiche2009modeling}
A.~Darwiche.
\newblock \emph{Modeling and {R}easoning with {B}ayesian {N}etworks}.
\newblock Cambridge University Press, 2009.

\bibitem[de~Salvo~Braz(2007)]{Braz07}
R.~de~Salvo~Braz.
\newblock \emph{Lifted {F}irst-{O}rder {P}robabilistic {I}nference}.
\newblock PhD thesis, Ph. D. Dissertation, University of Illinois at Urbana
  Champaign, 2007.

\bibitem[Dilkas and Belle(2023)]{dilkas2023synthesising}
P.~Dilkas and V.~Belle.
\newblock Synthesising recursive functions for first-order model counting:
  Challenges, progress, and conjectures.
\newblock In \emph{20th International Conference on Principles of Knowledge
  Representation and Reasoning}, 2023.

\bibitem[Gehrke(2021)]{Geh21a}
M.~Gehrke.
\newblock \emph{{Taming Exact Inference in Temporal Probabilistic Relational
  Models}}.
\newblock PhD thesis, University of L\"ubeck, 2021.

\bibitem[Gehrke et~al.(2018{\natexlab{a}})Gehrke, Braun, and
  M\"oller]{GehBrMo18c}
M.~Gehrke, T.~Braun, and R.~M\"oller.
\newblock {Towards Preventing Unnecessary Groundings in the Lifted Dynamic
  Junction Tree Algorithm}.
\newblock In \emph{Proceedings of {KI} 2018: Advances in Artificial
  Intelligence}, pages 38--45. Springer, 2018{\natexlab{a}}.

\bibitem[Gehrke et~al.(2018{\natexlab{b}})Gehrke, Braun, and
  M\"oller]{GehBrMo18d}
M.~Gehrke, T.~Braun, and R.~M\"oller.
\newblock {Answering Multiple Conjunctive Queries with the Lifted Dynamic
  Junction Tree Algorithm}.
\newblock In \emph{Proceedings of the AI 2018: Advances in Artificial
  Intelligence}, pages 543--555. Springer, 2018{\natexlab{b}}.

\bibitem[Gehrke et~al.(2018{\natexlab{c}})Gehrke, Braun, and
  M\"oller]{GehBrMo18e}
M.~Gehrke, T.~Braun, and R.~M\"oller.
\newblock {Preventing Unnecessary Groundings in the Lifted Dynamic Junction
  Tree Algorithm}.
\newblock In \emph{Proceedings of the AI 2018: Advances in Artificial
  Intelligence}, pages 556--562. Springer, 2018{\natexlab{c}}.

\bibitem[Gehrke et~al.(2018{\natexlab{d}})Gehrke, Braun, and
  M{\"o}ller]{gehrke2018ldjt}
M.~Gehrke, T.~Braun, and R.~M{\"o}ller.
\newblock {Lifted Dynamic Junction Tree Algorithm}.
\newblock In \emph{Proceedings of the 23rd International Conference on
  Conceptual Structures}, pages 55--69. Springer, 2018{\natexlab{d}}.

\bibitem[Gehrke et~al.(2019)Gehrke, Braun, and M\"oller]{GehBrMo19a}
M.~Gehrke, T.~Braun, and R.~M\"oller.
\newblock {Relational Forward Backward Algorithm for Multiple Queries}.
\newblock In \emph{Proceedings of the 32nd International Florida Artificial
  Intelligence Research Society Conference {(FLAIRS-32)}}, pages 464--469. AAAI
  Press, 2019.

\bibitem[Gehrke et~al.(2020)Gehrke, M\"oller, and Braun]{GehMoBr20}
M.~Gehrke, R.~M\"oller, and T.~Braun.
\newblock {Taming Reasoning in Temporal Probabilistic Relational Models}.
\newblock In \emph{Proceedings of the 24th European Conference on Artificial
  Intelligence (ECAI 2020)}, pages 2592--2599. {IOS} Press, 2020.

\bibitem[Geier and Biundo(2011)]{geier2011approximate}
T.~Geier and S.~Biundo.
\newblock Approximate {O}nline {I}nference for {D}ynamic {M}arkov {L}ogic
  {N}etworks.
\newblock In \emph{Proceedings of the 23rd IEEE International Conference on
  Tools with Artificial Intelligence}, pages 764--768. IEEE, 2011.

\bibitem[Kazemi et~al.(2016)Kazemi, Kimmig, Van~den Broeck, and
  Poole]{kazemi2016new}
S.~M. Kazemi, A.~Kimmig, G.~Van~den Broeck, and D.~Poole.
\newblock New liftable classes for first-order probabilistic inference.
\newblock In \emph{Advances in Neural Information Processing Systems}, 2016.

\bibitem[Kuželka(2023)]{ijcai2023p801}
O.~Kuželka.
\newblock Counting and sampling models in first-order logic.
\newblock In \emph{Proceedings of the Thirty-Second International Joint
  Conference on Artificial Intelligence, {IJCAI-23}}, pages 7020--7025.
  International Joint Conferences on Artificial Intelligence Organization, 8
  2023.

\bibitem[Marwitz et~al.(2023)Marwitz, M\"oller, and Gehrke]{MaMoGe23}
F.~A. Marwitz, R.~M\"oller, and M.~Gehrke.
\newblock { PETS: Predicting Efficiently using Temporal Symmetries in Temporal
  PGMs}.
\newblock In \emph{Proceedings of the Seventeenth European Conference on
  Symbolic and Quantitative Approaches to Reasoning with Uncertainty
  (ECSQARU-23)}. Springer, 2023.

\bibitem[Milch et~al.(2008)Milch, Zettlemoyer, Kersting, Haimes, and
  Kaelbling]{milch2008lifted}
B.~Milch, L.~S. Zettlemoyer, K.~Kersting, M.~Haimes, and L.~P. Kaelbling.
\newblock Lifted {P}robabilistic {I}nference with {C}ounting {F}ormulas.
\newblock In \emph{AAAI08 Proceedings of the 23rd National Conference on
  Artificial Intelligence - Volume 2}, pages 1062--1068. AAAI Press, 2008.

\bibitem[Murphy(2002)]{Murphy:2002:DBN}
K.~P. Murphy.
\newblock \emph{{Dynamic Bayesian Networks: Representation, Inference and
  Learning}}.
\newblock PhD thesis, University of California, Berkeley, 2002.

\bibitem[Niepert and Van~den Broeck(2014)]{niepert2014tractability}
M.~Niepert and G.~Van~den Broeck.
\newblock {Tractability through Exchangeability: A New Perspective on Efficient
  Probabilistic Inference}.
\newblock In \emph{AAAI14 Proceedings of the Twenty-Eighth AAAI Conference on
  Artificial Intelligence}, pages 2467--2475. AAAI Press, 2014.

\bibitem[Papai et~al.(2012)Papai, Kautz, and Stefankovic]{papai2012slice}
T.~Papai, H.~Kautz, and D.~Stefankovic.
\newblock Slice {N}ormalized {D}ynamic {M}arkov {L}ogic {N}etworks.
\newblock In \emph{NIPS12 Proceedings of the 25th International Conference on
  Neural Information Processing Systems - Volume 2}, pages 1907--1915. Curran
  Associates Inc., 2012.

\bibitem[Poole(2003)]{poole2003first}
D.~Poole.
\newblock First-order probabilistic inference.
\newblock In \emph{IJCAI03 Proceedings of the 18th International Joint
  Conference on Artificial Intelligence}, pages 985--991. Morgan Kaufmann
  Publishers Inc., 2003.

\bibitem[Taghipour(2013)]{Taghipour2013lifted}
N.~Taghipour.
\newblock \emph{Lifted Probabilistic Inference by Variable Elimination}.
\newblock PhD thesis, Ph. D. Dissertation, KU Leuven, 2013.

\bibitem[Taghipour et~al.(2013{\natexlab{a}})Taghipour, Davis, and
  Blockeel]{taghipour2013first}
N.~Taghipour, J.~Davis, and H.~Blockeel.
\newblock First-order {D}ecomposition {T}rees.
\newblock In \emph{NIPS13 Proceedings of the 26th International Conference on
  Neural Information Processing Systems - Volume 1}, pages 1052--1060. Curran
  Associates Inc., 2013{\natexlab{a}}.

\bibitem[Taghipour et~al.(2013{\natexlab{b}})Taghipour, Fierens, Davis, and
  Blockeel]{TagFiDaBl13}
N.~Taghipour, D.~Fierens, J.~Davis, and H.~Blockeel.
\newblock {{Lifted Variable Elimination: Decoupling the Operators from the
  Constraint Language}}.
\newblock \emph{Journal of Artificial Intelligence Research}, 47(1):\penalty0
  393--439, 2013{\natexlab{b}}.

\bibitem[Taghipour et~al.(2013{\natexlab{c}})Taghipour, Fierens, Van~den
  Broeck, Davis, and Blockeel]{taghipour2013completeness}
N.~Taghipour, D.~Fierens, G.~Van~den Broeck, J.~Davis, and H.~Blockeel.
\newblock {Completeness Results for Lifted Variable Elimination}.
\newblock In \emph{Artificial Intelligence and Statistics}, pages 572--580,
  2013{\natexlab{c}}.

\bibitem[Van~Bremen and Ku{\v{z}}elka(2021)]{van2021faster}
T.~Van~Bremen and O.~Ku{\v{z}}elka.
\newblock Faster lifting for two-variable logic using cell graphs.
\newblock In \emph{Uncertainty in Artificial Intelligence}, pages 1393--1402.
  PMLR, 2021.

\bibitem[Van~Bremen and Ku{\v{z}}elka(2023)]{van2023lifted}
T.~Van~Bremen and O.~Ku{\v{z}}elka.
\newblock Lifted inference with tree axioms.
\newblock \emph{Artificial Intelligence}, 324:\penalty0 103997, 2023.

\bibitem[Van~den Broeck(2011)]{broeck2011completeness}
G.~Van~den Broeck.
\newblock {On the Completeness of First-Order Knowledge Compilation for Lifted
  Probabilistic Inference}.
\newblock In \emph{NIPS11 Proceedings of the 24th International Conference on
  Neural Information Processing Systems}, pages 1386--1394. Curran Associates
  Inc., 2011.

\end{thebibliography}

\appendix 

In this appendix, we try to compile all cited definitions, operators, and algorithms, as a service for the reviewers to have everything in one place.
As this appendix is solely a composition of previously published work, we often just copy the text verbatim from the corresponding papers to not alter their meaning.

\section{LDJT}\label{app:ldjt}

This section contains definitions from \cite{gehrke2018ldjt,braun2017preventing,GehBrMo18c,GehBrMo18d,GehBrMo18e,GehBrMo19a,Geh21a}

\subsection{FO Jtree}

\begin{definition}[Parcluster, FO jtree]
	Let $\mathbf{X}$ be a set of logical variables, $\mathbf{A}$ a set of PRVs with $lv(\mathbf{A}) \subseteq \mathbf{X}$, and $(\mathcal{X}, C_\mathcal{X})$ a constraint on $\mathbf{X}$.
	Then, $\forall \mathbf{x} \in C_\mathcal{X} : \mathbf{A}_{| C}$ denotes a \emph{parcluster}, substituting $\mathbf{X}$ in $\mathbf{A}$ with $\mathbf{x}$.
	We write $\mathbf{A}_{| (\mathcal{X}, C_\mathcal{X})}$ for short.
	We omit $|(\mathcal{X}, C_\mathcal{X})$ if the constraint is $\top$.
	
	An \emph{FO jtree} for a model $G$ is a cycle-free graph $J = (V, E)$, where $V \subseteq 2^{rv(G)}$ is the set of nodes and $E \subseteq \{\{i, j\} | i, j \in V, i \not= j\}$ the set of edges.
	Each node in $V$ is a parcluster $\mathbf{C}^i$.
	$J$ must satisfy three properties:
	\begin{enumerate}[(i)]
		\item $\forall \mathbf{C}^i \in V : \mathbf{C}^i \subseteq rv(G)$.
		\item $\forall g  \in G : \exists \mathbf{C}^i \in V : rv(g) \subseteq \mathbf{C}^i$.
		\item If $\exists A \in rv(G) : A \in \mathbf{C}^i \wedge A \in \mathbf{C}^j$, then $\forall \mathbf{C}^k$ on the path between $\mathbf{C}^i$ and $\mathbf{C}^j : A \in \mathbf{C}^k$ (running intersection property).
	\end{enumerate}
	An FO jtree is \emph{minimal} if by removing a PRV from a parcluster, the FO jtree ceases to be an FO jtree, i.e., it no longer fulfils all properties.
	The set $\mathbf{S}^{ij}$, called \emph{separator} of edge $\{i, j\} \in E$, is defined by $\mathbf{C}^i \cap \mathbf{C}^j$.
	The term $nbs(i)$ refers to the neighbours of node $i$, defined by $\{j | \{i,j\} \in E\}$.
	Each $\mathbf{C}^i \in V$ has a \emph{local model} $G^i$ and $\forall g \in G_i : rv(g) \subseteq \mathbf{C}^i$.
	The local models $G^i$ partition $G$.
\end{definition}

\begin{algorithm}[]
    \caption{FO Jtree Construction for a PDM $(G_0,G_\rightarrow)$}
    \label{alg:construction}
    \begin{algorithmic}
        \Function{DFO-JTREE}{$G_0,G_\rightarrow$}
            \State $\mathrlap{\mathbf{I}_t}\hphantom{g^I_{t-1}} :=$ Set of interface \acp{prv} for time slice $t$ 
            \State $\mathrlap{g^I_0}\hphantom{g^I_{t-1}} :=$ Parfactor for $\mathbf{I}_0$ 
            \State $\mathrlap{G_0}\hphantom{g^I_{t-1}} :=  g^I_0 \cup G_0$ 
            \State $\mathrlap{J_0}\hphantom{g^I_{t-1}} :=$ Construct minimised FO jtree for $G_0$ and remove $g^I_0$
            \State $g^I_{t-1} := $ Parfactor for $\mathbf{I}_{t-1}$
            \State $\mathrlap{g^I_t}\hphantom{g^I_{t-1}} :=$ Parfactor for $\mathbf{I}_t$
            \State $\mathrlap{F_t}\hphantom{g^I_{t-1}} := \{\phi(\mathcal{A})_{| C}  \in G_\rightarrow  \mid  \forall A \in \mathcal{A} : A \notin \mathbf{A}_t   \}$
            \State $\mathrlap{G_t}\hphantom{g^I_{t-1}} := (g^I_{t-1}  \cup g^I_t \cup F_t)$
            \State $\mathrlap{J_t}\hphantom{g^I_{t-1}} :=$ Construct minimised FO jtree for $G_t$ and remove $g^I_{t-1}$ as well as $g^I_{t}$\\
            \Return $(J_0,J_t, \mathbf{I}_t)$
        \EndFunction
    \end{algorithmic}  
\end{algorithm}

\begin{definition}[Interface \acp{prv}]
    The forward interface is defined as $\mathbf{I}_{t} = \{A_{t}^i \mid \exists \phi(\mathcal{A})_{| C }\in G :  A_{t}^i \in \mathcal{A} \wedge \exists A_{t+1}^j \in \mathcal{A}\}$.
    The set of non-interface \acp{prv} is $\mathbf{N}_t = \mathbf{A}_t \setminus \mathbf{I}_t$.
\end{definition}

\subsection{Message Passing}

The last step to prepare an \ac{fojt} before query answering is message passing.
An idea of \ac{ljt} is to use submodels to answer queries.
To be more precise, \ac{ljt} uses parclusters to answer queries.
In general, one possibility to answer queries is to compute a full joint distribution and then eliminate all none query terms, i.e, multiply all parfactors and sum out all non-query terms.
However, in the current state of an \ac{fojt}, the parclusters do not necessarily hold all state descriptions of their corresponding \acp{prv}, i.e., the full joint distribution for each parcluster only consists of its assigned parfactors.
To be able to answer queries correctly, a parcluster $\mathbf{C}^i$ has to query its neighbours for their state descriptions of the \acp{prv} from $\mathbf{C}^i$, i.e., a parcluster asks for the joint distribution of its \acp{prv} from other parclusters.
The neighbours of $\mathbf{C}^i$ might in turn need to query their neighbours to answer the query. 
Such recursive querying can be efficiently implemented using dynamic programming.
To directly distribute all state descriptions, \ac{ljt} performs a so-called message pass. 
Each parcluster has local state descriptions, due to the parfactors and evidence assigned to them.
Only after local state descriptions of each parcluster is distributed through the \ac{fojt}, \ac{ljt} can use any parcluster the contains the query term to answer the query.
Message passing consists of an \emph{inbound} and an \emph{outbound} pass.
After a message pass, each parcluster obtains the complete state descriptions of its \acp{prv} instead of each parcluster querying for the partial state descriptions from other parclusters.
During the messages pass, the influences of a \ac{prv} to other \acp{prv} are distributed throughout the parclusters.
To compute a message, \ac{ljt} eliminates, i.e., applies lifted summing out to, all non-separator \acp{prv} from the local model and received messages of the parcluster.
To calculate a message, incoming messages from the designated receiver of the message to be calculated are ignored.
A message is a \ac{pf} as it contains \acp{prv}, a mapping of these \acp{prv} to potentials, and a constraint.
After a complete message pass, each parcluster has all the state descriptions required to answer queries about its \acp{prv}.

\subsection{Preventing Groundings (Fusion and Extension)}\label{app:prev}

\ac{ljt} applies three tests to check whether groundings occur during message passing.
The first test checks if \ac{ljt} can apply lifted summing out.
In case \ac{ljt} can apply lifted summing out on a \ac{prv} to calculate a message, the \ac{prv} cannot cause unnecessary groundings during message passing of \ac{ljt}.
In case \ac{ljt} cannot apply lifted summing out on a \ac{prv}, the second test checks to whether groundings can be prevented by count-converting.
Unfortunately, even if a count-conversion prevents unnecessary grounding in a parcluster, the count-conversion can lead to groundings in another parcluster.
The third test validates that a count-conversion will not result in groundings in another parcluster.
\acused{lv}
Now, we present the problem and the checks formally, and afterwards we illustrate the checks using $G^{ex}$.
During message passing, a parcluster $\mathbf{C}^i = \mathbf{A}^i_{|C^i}$ sends a message $m^{ij}$ containing the \acp{prv} of the separator $\mathbf{S}^{ij}$ to parcluster $\mathbf{C}^j$.
To calculate the message $m^{ij}$, \ac{ljt} eliminates the parcluster \acp{prv} not being part of the separator, i.e., $\mathbf{A}^{ij} := \mathcal{A}^i \setminus \mathbf{S}^{ij}$, from the local model and all messages received from other nodes other than $j$, i.e., $G^\prime := G^i \cap \{m^{il}\}_{l \neq j}$. 
To eliminate $A \in \mathbf{A}^{ij}$ by lifted summing out from $G^\prime$, we replace all \acp{pf} $g \in G^\prime$ that include $A$ with a \ac{pf} $g^E = \phi(\mathcal{A}^E)_{|C^E}$ that is the lifted product, i.e., the multiplication of the \acp{pf} that include $A$. 
Let $\mathbf{S}^{ij^E} := \mathbf{S}^{ij} \cap \mathcal{A}^E$ be the set of \acp{prv} in the separator that occur in $g^E$. 
For lifted message calculation, it necessarily has to hold $\forall S \in \mathbf{S}^{ij^E}$,
\begin{equation}\label{eq:1}
    lv(S) \subseteq lv(A).
\end{equation}
Otherwise, $A$ does not include all \acp{lv} in $g^E$. 

A count conversion may induce \cref{eq:1} for a particular $S$ if
\begin{equation}\label{eq:2}
	\left(lv(S) \setminus lv(A)\right) \text{ is count-convertible in } g^E.
\end{equation}
In case \cref{eq:2} holds, \ac{ljt} count-converts $L$, yielding a (P)CRV in $m^{ij}$, otherwise, \ac{ljt} grounds.

Unfortunately, a (P)CRV can lead to groundings in another parcluster.
Hence, count-conversion helps in preventing a grounding if all following messages can handle the resulting (P)CRV. 
Formally, for each node $k$ receiving $S$ as a (P)CRV with counted \ac{lv} $L$, it has to hold for each neighbour $n$ of $k$ that

\begin{equation}\label{eq:3}
    S \in \mathbf{S}^{kn} \vee\text{ L count-convertible in }g^S.
\end{equation}

\ac{ljt} adjusts the elimination order in case the checks determine that groundings would occur by message passing between these two parclusters, which is the case
\begin{enumerate}
    \item either if \cref{eq:1} and \cref{eq:2} do not hold
    \item or if \cref{eq:1} does not hold, \cref{eq:2} holds, and \cref{eq:3} does not hold.
\end{enumerate}
To adjust the elimination order, \ac{ljt} applies a so-called \emph{fusion} operator to two parclusters, which merges these two parclusters. 
The idea behind fusing is twofold.
On the one hand, fusing the parclusters changes the elimination order by reducing restrictions imposed on the elimination order by an \ac{fojt} and thereby, leads to preventing unnecessary grounding. 
On the other hand, by fusing the parclusters, \ac{ljt} does not have to recompute incoming messages to the fused parcluster.

\begin{algorithm}[]
    \caption{Preventing Groundings for FO Jtree $(J_0,J_t)$ during a Forward Pass}
    \label{alg:preventingForward}
    \begin{algorithmic}
        \Function{PreventForwardGroundings}{$J_0,J_t$}
            \State $\mathbf{C}^i := J_0(\text{out-cluster})$
            \State $\mathbf{C}^j := J_1(\text{in-cluster})$ \Comment $J_t$ instantiated for $t=1$
            \State $\mathbf{A}^{\alpha_0} := \mathbf{C}^i\setminus \mathbf{S}^{ij}$
            \For{$A \in \mathbf{A}^{\alpha_0}$}
                \If{$A$ induces groundings} \Comment Based on \cref{eq:1,eq:2,eq:3}
                    \State Add $A$ to $\mathbf{C}^j$
                \EndIf
            \EndFor
            
            \State $\mathbf{C}^i := J_{t-1}(\text{out-cluster})$
            \State $\mathbf{C}^j := J_{t}(\text{in-cluster})$
            \State $\mathbf{A}^{\alpha_{t-1}} := \mathbf{C}^i\setminus \mathbf{S}^{ij}$
            \For{$A \in \mathbf{A}^{\alpha_{t-1}}$}
                \If{$A$ induces groundings} \Comment Based on \cref{eq:1,eq:2,eq:3}
                    \State Add $A$ to $\mathbf{C}^j$
                \EndIf
            \EndFor
            \State Prevent unnecessary groundings for $J_{t}$
            \State \Return $J_{t}$
        \EndFunction
    \end{algorithmic}  
\end{algorithm}

\begin{algorithm}
    \caption{Preventing Groundings for FO Jtrees $(J_0,J_t)$ during a Backward Pass}
    \label{alg:preventingBackward}
    \begin{algorithmic}
        \Function{PreventBackwardGroundings}{$J_0,J_t$}
            \State $\mathbf{C}^i := J_{t-1}(\text{out-cluster})$
            \State $\mathbf{C}^j := J_{t}(\text{in-cluster})$
            \State $\mathbf{A}^{\beta_{t}} := \mathbf{C}^j\setminus \mathbf{S}^{ji}$
            \For{$A \in \mathbf{A}^{\beta_{t}}$}
                \If{$A$ induces groundings}
                    \State Add $A$ to $i$ and $J_{0}(\text{out-cluster})$
                \EndIf
            \EndFor
            \State Prevent unnecessary groundings for $J_{0}$ 
            \State Prevent unnecessary groundings for $J_{t}$ 
            \State \Return $(J_0, J_t)$
        \EndFunction
    \end{algorithmic}  
\end{algorithm}

\subsection{LDJT}

\begin{algorithm}[H]
    \caption{LDJT Alg. for PDM $(G_0,G_\rightarrow)$, Queries $\{\mathbf{Q}\}_{t=0}^T$, Evidence $\{\mathbf{E}\}_{t=0}^T$}
    \label{alg:LDJT}
    \begin{algorithmic}
        \Procedure{LDJT}{$G_0,G_\rightarrow, \{\mathbf{Q}\}_{t=0}^T, \{\mathbf{E}\}_{t=0}^T$}
            \State $t := 0$
            \State $(J_0,J_t, \mathbf{I}_t) : =$  \Call  {DFO-JTREE}{$G_0,G_\rightarrow$ }
            \State $(J_0,J_t) : =$  \Call {PreventForwardGroundings}{$J_0,J_t$}
            \State $(J_0,J_t) : =$  \Call {PreventBackwardGroundings}{$J_0,J_t$}
            \While{$t \neq T+1$}
            \State $J_{t} :=$  \Call {LJT.EnterEvidence} {$J_{t}, \mathbf{E}_t$}
            \State $J_{t} :=$  \Call {LJT.PassMessages} {$J_{t}$}
            \For{$q_{\pi} \in \mathbf{Q}_t$}
            \State \Call {AnswerQuery}{$J_0, J_t, q_{\pi}, \mathbf{I}_t, \alpha, t$}
            \EndFor
            \State $(J_t,t, \alpha[t-1]) :=$  \Call {ForwardPass}{$J_0, J_t, t, \mathbf{I}_t$}
        \EndWhile
        \EndProcedure
    \end{algorithmic}  
    \hrulefill
    \begin{algorithmic}
        \Procedure{AnswerQuery}{$J_0, J_t, q_{\pi}, \mathbf{I}_t, \alpha, t$} 
        \While{$t \neq \pi$}
            \If{$t>\pi$}
                \State $(J_t, t) :=$  \Call {BackwardPass}{$J_0, J_t,  \mathbf{I}_t, \alpha[t-1], t$ }
            \Else
                \State $(J_t, t, \_) :=$  \Call {ForwardPass}{$J_0, J_t, \mathbf{I}_t, t$}                    
            \EndIf
        \State  \Call {LJT.PassMessages}{$J_t$}
        \EndWhile
        \State \textbf{print}  \Call {LJT.AnswerQuery}{$J_t, q_{\pi}$}
        \EndProcedure
    \end{algorithmic}
    \hrulefill
    \begin{algorithmic}
        \Function{ForwardPass}{$J_0, J_t, \mathbf{I}_t, t$}
            \State $\alpha_t := \sum_{J_t(\text{out-cluster}) \setminus \mathbf{I}_t} J_t(\text{out-cluster})$ 
            \State $t := t+1$
            \State $J_t(\text{in-cluster}) := \alpha_{t-1} \cup J_t(\text{in-cluster})$
            \State\Return $(J_t, t, \alpha_{t-1})$      
        \EndFunction
    \end{algorithmic}
    \hrulefill
    \begin{algorithmic}
        \Function{BackwardPass}{$J_0, J_t, \mathbf{I}_t, \alpha_{t-1}, t$}
            \State $\beta_t := \sum_{J_t(\text{in-cluster}) \setminus \mathbf{I}_t} (J_t(\text{in-cluster}) \setminus \alpha_{t-1})$ 
            \State $t := t-1$
            \State $J_t(\text{out-cluster}) := \beta_{t+1} \cup J_t(\text{out-cluster})$
            \State\Return $(J_t, t)$      
        \EndFunction
    \end{algorithmic}
\end{algorithm}

\section{Operators of LVE}\label{app:lve}

This section contains operators from \cite{Taghipour2013lifted}.

\subsection{Helper Functions}

\begin{definition}{(Count Function)}
Given a constraint $G_\mathbf{X}$, for any $\mathbf{Y} \subseteq \mathbf{X}$ and $\mathbf{Z} \subseteq \mathbf{X} - \mathbf{Y}$, the function $\text{\textsc{count}}_{\mathbf{Y} | \mathbf{Z}} : C_\mathbf{X} \rightarrow \mathbb{N}$ is defined as follows:
\begin{align*}
	\text{\textsc{count}}_{\mathbf{Y}|\mathbf{Z}}(t) = | \pi_\mathbf{Y}(\sigma_{\mathbf{Z}=\pi_\mathbf{Z}(t)}(C_\mathbf{X})) |
\end{align*}
That is, for any tuple $t$, this function tells us how many values for $\mathbf{Y}$ co-occur with $t$'s value for $\mathbf{Z}$ in the constraint.
We define $\text{\textsc{count}}_{\mathbf{Y}|\mathbf{Z}}(t) = 1$ for $\mathbf{Y} = \emptyset$.
\end{definition}

\begin{definition}{(Count-normalised Constraint)}
For any constraint $C_\mathbf{X}$, $\mathbf{Y} \subseteq \mathbf{X}$ and $\mathbf{Z} \subseteq \mathbf{X} - \mathbf{Y}$, $\mathbf{Y}$ is count-normalised w.r.t. $\mathbf{Z}$ in $C_\mathbf{X}$ iff
\begin{align*}
	\exists n \in \mathbb{N} : \forall t \in C_\mathbf{X} : \text{\textsc{count}}_{\mathbf{Y}|\mathbf{Z}}(t) = n.
\end{align*}
If such an $n$ exists, we call it the conditional count of $\mathbf{Y}$ given $\mathbf{Z}$ in $C_\mathbf{X}$, and denote it $\text{\textsc{count}}_{\mathbf{Y}|\mathbf{Z}}(C_\mathbf{X})$.
\end{definition}

\begin{definition}{(Multiplicity of a Histogram)}
The multiplicity of a histogram $h = \{(r_1,n_2), \dots, (r_k,n_k)\}$ is a multinomial coefficient, defined as
\begin{align*}
	\text{\textsc{Mul}}(h) = \frac{n!}{\prod_{i=1}^k n_i!}
\end{align*}
As multiplicities should only be taken into account for (P)CRVs, never for regular PRVs, we define for each PRV $A$ and for each value $v \in range(A)$: $\text{\textsc{Mul}}(A,v) = 1$ if $A$ is a regular PRV, and $\text{\textsc{Mul}}(A,v) = \text{\textsc{Mul}}(v)$ if $A$ is a PCRV.
\end{definition}

\begin{definition}{(Splitting on Overlap)}
Splitting a constraint $C_1$ on its $\mathbf{Y}$-overlap with $C_2$, denoted $C_1 /_\mathbf{Y} C_2$, partitions $C_1$ into two subsets, containing all tuples for which the $\mathbf{Y}$ part occurs or does not occur, respectively, in $C_2$.
$C_1 /_\mathbf{Y} C_2 = \{\{ t \in C_1 | \pi_\mathbf{Y}(t) \in \pi_\mathbf{Y}(C_2) \},\{ t \in C_1 | \pi_\mathbf{Y}(t) \notin \pi_\mathbf{Y}(C_2) \}\}$
\end{definition}

\begin{definition}{(Parfactor Partitioning)}
Given a parfactor $g = \phi(\mathcal{A}) | C$ and a partition $\mathbb{C} = \{C_i\}_{i=1}^n$ of $C$, $\text{\textsc{partition}}(g,\mathbb{C}) = \{\phi(\mathcal{A})|C_i\}_{i=1}^n$
\end{definition}

\begin{definition}{(Group-by)}
Given a constraint $C$ and a function $f : C \rightarrow R$, $\text{\textsc{Group-By}}(C,f) = C/ \sim_f$, with $x \sim_f y \Leftrightarrow f(x) = f(y)$ and $/$ denoting set quotient.
That is, $\text{\textsc{Group-By}}(C,f)$ partitions $C$ into subsets of elements that have the same result for $f$.
\end{definition}

\begin{definition}{(Joint-count)}
Given a constraint $C$ over variables $\mathbf{X}$, partitioned into $\{C_1,C_2\}$, and a counted logvar $X \in \mathbf{X}$; then for any $t\in C$, with $L = \mathbf{X}\backslash \{X\}$ and $l = \pi_L(t)$,
\begin{align*}
	\text{\textsc{joint-count}}_{X,\{C_1,C_2\}}(t) = (|\pi_X(\sigma_{L=l}(C_1))|,|\pi_X(\sigma_{L=l}(C_2))|).
\end{align*}
\end{definition}

\subsection{Operators}

\begin{operator}[H]
\caption{Lifted Mulltiplication. The definition assumes, without loss of generality that the logvars in the parfactors are standardized apart, i.e., the two parfactors do not share variable names (this can be achieved by renaming logvars).}
\label{op:mult}
	\textbf{Operator} \textsc{multiply}\\
	\textbf{Inputs:}
	\begin{algorithmic}[0]
		\State (1) $g_1 = \phi_1(\mathcal{A}_1) | C_1$: a parfactor in $G$
		\State (2) $g_2 = \phi_2(\mathcal{A}_2) | C_2$: a parfactor in $G$
		\State (3) $\theta = \{ \mathbf{X}_1 \rightarrow \mathbf{X}_2 \}$: an alignment between $g_1$ and $g_2$
	\end{algorithmic}
	\textbf{Preconditions:}
	\begin{algorithmic}[0]
		\State (1) for $i = 1, 2: \mathbf{Y}_i = logvar(\mathcal{A}_i) \setminus \mathbf{X}_i$ is count-normalised w.r.t. $\mathbf{X}_i$ in $C_i$
	\end{algorithmic}
	\textbf{Output:} $\phi(\mathcal{A}) | C$, with
	\begin{algorithmic}[0]
		\State (1) $C = \rho_\theta(C_1) \bowtie C_2$
		\State (2) $\mathcal{A} = \mathcal{A}_1\theta \cup \mathcal{A}_2$, and
		\State (3) for each valuation $\mathbf{a}$ of $\mathcal{A}$, with $\mathbf{a}_1 = \pi_{\mathcal{A}_1\theta}(\mathbf{a})$ and $\mathbf{a}_2 = \pi_{\mathcal{A}_2}(\mathbf{a})$:\\
		$\phi(\mathbf{a}) = \phi_1^{\frac{1}{r_2}}(\mathbf{a}_1)\cdot \phi_2^{\frac{1}{r_1}}(\mathbf{a}_2)$, with $r_i = \text{\textsc{Count}}_{\mathbf{Y}_i | \mathbf{X}_i} (C_i)$
	\end{algorithmic}
	\textbf{Postcondition:} $G \equiv G \setminus \{g_1,g_2\} \cup \{ \text{\textsc{multiply}} (g_1,g_2,\theta)\}$
\end{operator}

\begin{operator}
\caption{Lifted Summing-out.}
\label{op:sum}
	\textbf{Operator} \textsc{sum-out}\\
	\textbf{Inputs:}
	\begin{algorithmic}[0]
		\State (1) $g = \phi(\mathcal{A}) | C$: a parfactor in $G$
		\State (2) $A_i$: an atom in $\mathcal{A}$, to be summed out from $g$
	\end{algorithmic}
	\textbf{Preconditions:}
	\begin{algorithmic}[0]
		\State (1) For all PRVs $\mathcal{V}$, other than $A_i |C$, in model $G$: $RV(\mathcal{V}) \cap RV(A_i|C) = \emptyset$
		\State (2) $A_i$ contains all the logvars $X \in logvar(\mathcal{A})$ for which $\pi_X(C)$ is not singleton.
		\State (3) $\mathbf{X}^{excl} = logvar(A_i) \cap logvar(\mathcal{A} \setminus A_i)$ is count-normalised w.r.t.\\ $\mathbf{X}^{com} = logvar(A_i) \cap logvar(\mathcal{A} \setminus A_i)$ in $C$
	\end{algorithmic}
	\textbf{Output:} $\phi'(\mathcal{A}') | C'$, such that
	\begin{algorithmic}[0]
		\State (1) $\mathcal{A}' = \mathcal{A}\setminus \{A_i\}$
		\State (2) $C' = \pi_{\mathbf{X}^{com}}(C)$
		\State (3) for each assignment $\mathbf{a}' = ( \dots,a_{i-1},a_{i+1},\dots)$ to $\mathcal{A}'$, \\ $\phi'( \dots,a_{i-1},a_{i+1},\dots) = \sum_{a_i \in range(A_i)} \text{\textsc{Mul}}(A_i,a_i)\phi( \dots,a_{i-1},a_{i+1},\dots)^r$\\ with $r = \text{\textsc{Count}}_{\mathbf{X}^{excl}|\mathbf{X}^{com}}(C)$
	\end{algorithmic}
	\textbf{Postcondition:} $\mathcal{P}_{G\setminus\{g\}\cup\{\text{\textsc{sum-out}}(g,A_i)\}} = \sum_{RV(A_i|C)}\mathcal{P}_G$
\end{operator}

\begin{operator}
\caption{Counting Conversion.}
\label{op:convert}
	\textbf{Operator} \textsc{count-convert}\\
	\textbf{Inputs:}
	\begin{algorithmic}[0]
		\State (1) $g = \phi(\mathcal{A}) | C$: a parfactor in $G$
		\State (2) $X$: a logvar in $logvar(\mathcal{A})$
	\end{algorithmic}
	\textbf{Preconditions:}
	\begin{algorithmic}[0]
		\State (1) there is exactly one atom $A_i \in \mathcal{A}$ with $X \in logvar(A_i)$
		\State (2) $X$ is count-normalised w.r.t. $logvar(\mathcal{A}) \setminus \{X\}$ in $C$
		\State (3) for all counted logvars $X^{\#}$ in $g$: $\pi_{X,X^{\#}}(C) = \pi_X \times \pi_{X^{\#}} (C)$
	\end{algorithmic}
	\textbf{Output:} $\phi'(\mathcal{A}') | C'$, such that
	\begin{algorithmic}[0]
		\State (1) $\mathcal{A}' = \mathcal{A}\setminus \{A_i\} \cup \{A'_i\}$ with $A'_i = \#_X[A_i]$
		\State (2) for each assignment $\mathbf{a}'$ to $\mathcal{A}'$ with $a'_i = h$: \\ $\phi'( \dots,a_{i-1},h,a_{i+1},\dots) = \prod_{a_i \in range(A_i)} \phi( \dots,a_{i-1},a_i,a_{i+1},\dots)^{h(a_i)}$\\ with $h(a_i)$ denoting the count of $a_i$ in histogram $h$
	\end{algorithmic}
	\textbf{Postcondition:} $G \equiv G \setminus \{g\} \cup \{ \text{\textsc{count-convert}}(g,X)\}$
\end{operator}

\begin{operator}
\caption{Splitting.}
\label{op:split}
	\textbf{Operator} \textsc{split}\\
	\textbf{Inputs:}
	\begin{algorithmic}[0]
		\State (1) $g = \phi(\mathcal{A}) | C$: a parfactor in $G$
		\State (2) $A = P(\mathbf{Y})$: an atom in $\mathcal{A}$
		\State (3) $A' = P(\mathbf{Y})|C'$ or $\#_Y[P(\mathbf{Y})]|C'$
	\end{algorithmic}
	\textbf{Output:} \textsc{partition}$(g,\mathbb{C})$, with $\mathbb{C} = C/ _{\mathbf{Y}}C'\setminus{\emptyset}$\\
	\textbf{Postcondition:} $G \equiv G \setminus \{g\} \cup \text{\textsc{split}}(g,A,A')$
\end{operator}

\begin{operator}
\caption{Expansion.}
\label{op:expand}
	\textbf{Operator} \textsc{expand}\\
	\textbf{Inputs:}
	\begin{algorithmic}[0]
		\State (1) $g = \phi(\mathcal{A}) | C$: a parfactor in $G$
		\State (2) $A = \#_X[P(\mathbf{X})]$: a counting formula in $\mathcal{A}$
		\State (3) $A' = P(\mathbf{X})|C'$ or $\#_Y[P(\mathbf{X})]|C'$
	\end{algorithmic}
	\textbf{Output:} $\{ g_i = \phi'_i(\mathcal{A}'_i) | C'_i \}_{i=1}^n$ where
	\begin{algorithmic}[0]
		\State (1) $C /_{\mathbf{X}}C' = \{C^{com},C^{excl}\}$
		\State (2) $\{C_1, \dots, C_n\} = \text{\textsc{group-by}}(C,\text{\textsc{joint-count}}_{X,C/_{\mathbf{X}}C'})$
		\State (3) for all $i$ where $C_i \bowtie C^{com} = \emptyset$ or $C_i \bowtie C^{excl} = \emptyset$: $\phi'_i = \phi, \mathcal{A}'_i = \mathcal{A}, C'_i = C_i$
		\State (4) for all other $i$:\\
		$C'_i = \pi_{logvar(\mathcal{A})}(C_i) \bowtie (\rho_{X\rightarrow X_{com}}(C^{com}) \bowtie \rho_{X\rightarrow X_{excl}}(C^{excl}))$,\\
		$\mathcal{A}'_i = \mathcal{A} \setminus \{A\} \cup \{A\theta_{com},A\theta_{excl}\}$ with $\theta_{com} = \{X \rightarrow X_{com}\}$, $\theta_{excl} = \{X \rightarrow X_{excl}\},$
		for each valuation $(\mathbf{1},h_{com},h_{excl})$ of $\mathcal{A}'_i$, $\phi'_i(\mathbf{l},h_{com},h_{excl}) = \phi(\mathbf{l},h_{com} \oplus h_{excl})$
	\end{algorithmic}
	\textbf{Postcondition:} $G \equiv G \setminus \{g\} \cup \text{\textsc{expand}}(g,A,A')$
\end{operator}

\begin{operator}
\caption{Counting Normalisation.}
\label{op:normal}
	\textbf{Operator} \textsc{count-normalise}\\
	\textbf{Inputs:}
	\begin{algorithmic}[0]
		\State (1) $g = \phi(\mathcal{A}) | C$: a parfactor in $G$
		\State (2) $\mathbf{Y}|\mathbf{Z}$: sets of logvars indicating the desired normalisation property in $C$
	\end{algorithmic}
	\textbf{Preconditions:}
	\begin{algorithmic}[0]
		\State (1) $\mathbf{Y} \subset logvar(\mathcal{A})$ and $\mathbf{Z} \subseteq logvar(\mathcal{A}) \setminus \mathbf{Y}$
	\end{algorithmic}
	\textbf{Output:} \textsc{partition}$(g,\text{\textsc{group-by}}(C,\text{\textsc{Count}}_{\mathbf{Y}|\mathbf{Z}}))$\\
	\textbf{Postcondition:} $G \equiv G \setminus \{g\} \cup \text{\textsc{count-normalise}}(g,\mathbf{Y}|\mathbf{Z})$
\end{operator}

\begin{operator}
\caption{Lifted Absorption.}
\label{op:absorb}
	\textbf{Operator} \textsc{absorb}\\
	\textbf{Inputs:}
	\begin{algorithmic}[0]
		\State (1) $g = \phi(\mathcal{A}) | C$: a parfactor in $G$
		\State (2) $A_i \in \mathcal{A}$ with $A_i = P(\mathbf{X})$ or $A_i = \#_{X_i}[P(\mathbf{X})]$
		\State (3) $g_E = \phi_E(P(\mathbf{X})) | C_E$: an evidence parfactor
		\State Let $\mathbf{X}^{excl} = \mathbf{X} \setminus logvar(\mathcal{A} \setminus A_i)$;
		\State $\mathbf{X}^{nce} = \mathbf{X}^{excl} \setminus \{X_i\}$ if $A_i = \#_{X_i}[P(\mathbf{X})]$, $\mathbf{X}^{excl}$ otherwise;
		\State $L' = logvar(\mathcal{A}) \setminus \mathbf{X}^{excl}$;
		\State $o = \text{the observed value for} P(\mathbf{X})$ in $g_E$
	\end{algorithmic}
	\textbf{Preconditions:}
	\begin{algorithmic}[0]
		\State (1) $RV(A_i | C_i) \subseteq RV(A_i|C_E)$
		\State (2) $\mathbf{X}^{nce}$ is count-normalised w.r.t. $L'$ in $C$.
	\end{algorithmic}
	\textbf{Output:} $\phi'(\mathcal{A}') | C'$, with
	\begin{algorithmic}[0]
		\State (1) $\mathcal{A}' = \mathcal{A}\setminus \{A_i\}$
		\State (2) $C' = \pi_{logvar(C)\setminus \mathbf{X}^{excl}}(C)$
		\State (3) $\phi'( \dots,a_{i-1},a_{i+1},\dots) = \phi( \dots,a_{i-1},e,a_{i+1},\dots)^r$ with $r = \text{\textsc{Count}}_{\mathbf{X}^{nce}|L'}(C)$ and\\
		with $e=o$ if $A_i = P(\mathbf{X})$ and \\
		$e$ a histogram with $e(o) = \text{\textsc{Count}}_{X_i|logvar(\mathcal{A})}(C)$, $e(.) = 0$ elsewhere, otherwise (namely if $A_i = \#_{X_i}[P(\mathbf{X})]$)
	\end{algorithmic}
	\textbf{Postcondition:} $G \cup \{g_E\} \equiv G \setminus \{g\} \cup \{g_E, \text{\textsc{absorb}}(g,A_i,g_E)\}$
\end{operator}

\begin{operator}[H]
\caption{Grounding.}
\label{op:ground}
	\textbf{Operator} \textsc{ground-logvar}\\
	\textbf{Inputs:}
	\begin{algorithmic}[0]
		\State (1) $g = \phi(\mathcal{A}) | C$: a parfactor in $G$
		\State (2) $X$: a logvar in $logvar(\mathcal{A})$
	\end{algorithmic}
	\textbf{Output:} \textsc{partition}$(g,\text{\textsc{group-by}}(C,\pi_X))$\\
	\textbf{Postcondition:} $G \equiv G \setminus \{g\} \cup \text{\textsc{ground-logvar}}(g,X)$
\end{operator}


\subsection{Generalised Counting Operators}\label{app:gc}

In this section, we mainly present definitions from \cite{Taghipour2013lifted}. 
First, we give an extended example of \acp{crv} and the provide the necessary operators.

\Acp{crv} are one important construct of \ac{lve} to enable lifted computations for, e.g., query answering, and \ac{ljt} uses \ac{lve} for its calculations and \ac{ldjt}, one of the main contributions of this dissertation, in turn uses \ac{ljt}.
We formally define a \ac{crv} next.
\begin{definition}[Parameterised CRV]
	Let $R(\mathbf{X})_{|C}$ denote a \ac{prv} under constraint $C$ where $lv(R(\mathbf{X})) = \{X\}$, meaning either $\mathbf{X}$ is a singleton set or other inputs to $R$ are constants.
	Then, the expression $\#_{X}[R(\mathbf{X})_{|C}]$ denotes a \emph{\ac{crv}}.
	Its range is the space of possible histograms.
	A \emph{histogram} $h$ is a set of tuples $\{(v^i, n^i)\}_{i = 1}^m$, $v^i \in \mathcal{R}(R(\mathbf{X}))$, $n^i \in \mathbb{N}$, $m = |\mathcal{R}(R(\mathbf{X}))|$, and $\sum_{i=1}^m n^i = |gr(X_{|C})|$. 
	A shorthand notation for the set of tuples is $[n^1, \dots, n^m]$. 
	As a function, $h$ takes a range value $v^i$ and returns the associated count $n^i$ from the tuple $(v^i, n^i)$.
	If $\{X\} \subset lv(P(\mathbf{X}))$, the CRV is a \emph{parameterised CRV (PCRV)} representing a set of CRVs.
	Since counting binds logical variables $X$, $lv(\#_{X} [R(\mathbf{X})]) = lv(R(\mathbf{X})) \setminus \{X\}$.
\end{definition}

\begin{example}[\ac{crv} as a compact encoding] \label{ex:crv:enc}
    Let us have a look at the \acp{rv} behind a boolean \ac{prv} $R(X)$ to illustrate \acp{crv}.
    Assuming we have a factor $\phi$, mapping the boolean arguments $R^1$, $R^2$, and $R^3$, for the three \acp{rv} behind $R(X)$, to potentials, which is the output of the factor, that is defined as follows:
	\begin{alignat}{5}
	& (\neg r^1, \neg r^2, \neg r^3)	&\mapsto 1,\ &  (\neg r^1, \neg r^2, r^3) &\mapsto 2,\ & (\neg r^1, r^2, \neg r^3) &\mapsto 2,\nonumber \\
	&  (\neg r^1, r^2, r^3) &\mapsto 3,\ & (r^1, \neg r^2, \neg r^3) & \mapsto 2,\ &	(r^1, \neg r^2, r^3) &\mapsto 3,\nonumber \\
	& (r^1, r^2, \neg r^3)		&\mapsto 3,\ &	 (r^1, r^2, r^3) & \mapsto 2 \label{eq:crv:gr}	
	\end{alignat}
    In all cases, three $false$ values map to $1$. 
	Two $false$ values and one $true$ value map to $2$.
	One $false$ value and two $true$ values map to $3$.
    Three $true$ values map to $2$.
	Now, assume a factor $\psi$ with one CRV and a logvar $L$, denoted as $\psi(\#_L[R(L)])$.
	Histograms range from $[0, 3]$ to $[3, 0]$, as there are $3$ interchangeable arguments, with the first position referring to $true$ and the second to $false$.
	The factor is defined as follows:
	\begin{align}
		[0, 3] \mapsto 1,\  [1, 2] \mapsto 2, \ [2, 1] \mapsto 3, \ [3, 0] \mapsto 2 \label{eq:crv}
	\end{align}
	$[2,1]$ maps to $2$ and $[1,2]$ maps to $3$.
	As the \acp{rv} are interchangeable, both factors encode the same information, but the \ac{crv} is a more compact representation.
	\Cref{eq:crv:gr} has $2^3 = 8$ mappings, \cref{eq:crv} has $\binom{3 + 2 - 1}{2 - 1} = 4$ mappings ($3$ \acp{rv}, each with $2$ range values), which is no longer exponential w.r.t.\ the number of original inputs.
\end{example}

The operators allow for 
\begin{itemize}
	\item count-converting logvars that appear in more than one PRV, 
	\item merging CRVs with counted logvars of the same domain into one CRV, and 
	\item merge-counting a PRV and a CRV with an inequality constraint into one CRV. 
\end{itemize}

\begin{operator}[H]
\caption{Generalised Count Conversion}
\label{op:mpe:gcc}
	\textbf{Operator} \textsc{count-convert}\\
	\textbf{Inputs:}
	\begin{compactenum}[(1)]
		\item a parfactor $g = \forall\mathbf{L} : \phi(\mathcal{A})_{| C}$
		\item a logvar $X \in logvar(\mathcal{A})$
	\end{compactenum}
	\textbf{Preconditions:}
	\begin{compactenum}[(1)]
		\item There is no counting formula in the set $\mathcal{A}_X = \{A \in \mathcal{A} | X \in lv(A)\}$.
		\item There is no counting formula $\gamma = \#_{X_i}[\dots] \in \mathcal{A}$, such that $X_i \not= X$ in $C_i$
	\end{compactenum}
	\textbf{Output:} $g' = \forall \mathbf{L}' : C'.\phi'(\mathcal{A}')$ such that
	\begin{compactenum}[(1)]
		\item $\mathbf{L}' = \mathbf{L} \setminus \{X\}$,
		\item $C'$ is the projection of $C$ on $\mathbf{L}'$
		\item $\mathcal{A}' = \mathcal{A} \setminus \mathcal{A}_X \cup \#_{X}[\mathcal{A}_X]$, and
		\item for each valuation $(h(.),\mathbf{a})$ to $\#_{X}[\mathcal{A}_X], \mathcal{A} \setminus \mathcal{A}_X):$
			\begin{equation*}
				 \phi'(h(.),\mathbf{a}) = \prod_\mathbf{a}' \in range(\mathcal{A}_X) \phi(\mathbf{a}';\mathbf{a})^{h(\mathbf{a}')} 
			\end{equation*}
	\end{compactenum}
	\textbf{Postcondition:} $G \equiv G \setminus \{g\} \cup \{ \text{\textsc{count-convert}}(g,X)\}$
\end{operator}
\vspace*{-5mm}
\begin{operator}[H]
\caption{Merging}
\label{op:mpe:merge}
	\textbf{Operator} \textsc{merge}\\
	\textbf{Inputs:}
	\begin{compactenum}[(1)]
		\item a parfactor $g = \forall\mathbf{L} :\phi(\mathcal{A})_{| C}$
		\item a pair of counting formulas $(\gamma_1, \gamma_2) = (\#_{X_1:C_1}[\mathcal{A}_1], \#_{X_2:C_2}[\mathcal{A}_2])$
	\end{compactenum}
	\textbf{Preconditions:}
	\begin{compactenum}[(1)]
		\item $gr(X_{1|C \wedge C_1}) = gr(X_{2|C \wedge C_2})$
	\end{compactenum}
	\textbf{Output:} $g' = \forall \mathbf{L} : \phi(\mathcal{A}')_{| C}$, such that
	\begin{compactenum}[(1)]
		\item $\mathcal{A}' = \mathcal{A}\setminus \{\gamma_1, \gamma_2\} \cup \{\#_{X_1}[\mathcal{A}_{12}]\}$, with $\mathcal{A}_{12} = \mathcal{A}_1 \cup \mathcal{A}_2 \theta$ and $\theta = \{X_2 \rightarrow X_1\}$
		\item for each valuation $(h(.), \mathbf{a})$, to $(\#_{X_1}[\mathcal{A}_{12}], \mathcal{A}\setminus \{\gamma_1, \gamma_2\}):$
			\begin{align*}
				\phi'(h,\mathbf{a}') = \phi(h_{[\mathcal{A}_1]},h_{[\mathcal{A}_2\theta]};\mathbf{a})
			\end{align*}
	\end{compactenum}
	\textbf{Postcondition:} $G \equiv G \setminus \{g\} \cup \{ \text{\textsc{merge}}(g,\gamma_1, \gamma_2)\}$
\end{operator}
Both operators are combined into one operator to merge and count a PRV into a CRV.

\begin{operator}[H]
\caption{Merge-counting}
\label{op:mpe:mcnt}
	\textbf{Operator} \textsc{merge-count}\\
	\textbf{Inputs:}
	\begin{compactenum}[(1)]
		\item a parfactor $g = \forall\mathbf{L} : \phi(\mathcal{A})_{| C}$
		\item a counting formula $\gamma = \#_{X_1:C_1}[\mathcal{A}_1] \in \mathcal{A}$
		\item a logvar $X_2 \in logvar(\mathcal{A})$, to merge-count into $\gamma$
	\end{compactenum}
	\textbf{Preconditions:}
	\begin{compactenum}[(1)]
		\item There is no counting formula in the set $\mathcal{A}_2 = \{A \in \mathcal{A} | X_2 \in logvar(A)\}$.
		\item $\gamma$ is the only  counting formula for whose counted logvar $X_1$ is in an inequality constraint with $X_2$.
	\end{compactenum}
	\textbf{Output:} $g = \forall \mathbf{L}' : \phi'(\mathcal{A}')_{|C'}$, such that:
	\begin{compactenum}[(1)]
		\item $\mathbf{L}' = \mathbf{L} \setminus \{X_2\}$
		\item $C'$ is the projection of C on $\mathbf{L}'$
		\item $\mathcal{A}' = \mathcal{A} \setminus \{\mathcal{A}_2, \gamma \} \cup \{\#_{X_1: C_{12}}[\mathcal{A}_{12}])$, with $\mathcal{A}_{12} = \mathcal{A}_1) \cup \mathcal{A}_2) \{X_2 \rightarrow X_1\}$, and  $C_{12} = C_1 \setminus (X_1 \neq X_2)$
		\item for each valuation $(h(.), \mathbf{a})$ to $(\#_{X_{1}:C_{12}}[\mathcal{A}_{12}], \mathcal{A'}\setminus \#_{X_{1}:C_{12}}[\mathcal{A}_{12}]):$
			\begin{align*}
				\phi'(h(.)\mathbf{a}') = \prod_{(\mathbf{a}_{12})\in range(\mathcal(A)_{12})} \phi(h^{-\mathbf{a}_1}_{[\mathcal{A}_{1}]},\mathbf{a}_2;\mathbf{a})^{h(\mathbf{a}_{12})}
			\end{align*}
			\begin{compactitem}
				\item $\mathbf{a}_i$ denotes the projection of the valuation of $\mathbf{a}_{12}$ on $\mathcal{A}_{i}\{X_2 \rightarrow X_1\}$,
				\item $h^{-\mathbf{r}}$ is such that $h^{-\mathbf{r}}(\mathbf{r}) = h(\mathbf{r}) -1$, and $h^{-\mathbf{r}}(\mathbf{r'}) = h(\mathbf{r'})$ for $\mathbf{r} \not= \mathbf{r}'$.
			\end{compactitem}
	\end{compactenum}
	\textbf{Postcondition:} $G \equiv G \setminus \{g\} \cup \{ \text{\textsc{merge-count}}(g, \gamma, X_2)\}$
\end{operator}

\section{LJT Complexity}\label{app:comp}
This section contains results from \cite{Bra20}.

\subsection{Complexity}\label{sec:complexity}
The complexity analyses of LVE and LJT mirror the complexity analyses of VE and JT, using the notion of tree width to characterise the complexity of inference.
Tree width refers to the ``largest'' cluster, i.e., the cluster with the most randvars, in a dtree or jtree \citep{Dar01}.
The largest cluster determines the worst case size a factor at a cluster can have, namely, if the factor has all cluster randvars as arguments.
The size of such a factor is $r^{w}$, $r$ being the largest range size in a model.
Eliminating all but one randvar is bounded by $O(r^{w})$ as there are at most $r^{w}-1$ sum-out operations.
JT has a complexity for a single query without preprocessing that is in $O(r^{w})$.
For VE, the complexity also depends on the overall number of eliminations for a single query.

In a worst case scenario, the lifted versions of VE and JT ground all logvars and perform inference at a propositional level for correct results.
The interesting case arises if a model allows for a lifted inference solution.
First, we characterise liftable models, which have a lifted inference solution and present the complexity results for LVE based on work by \cite{taghipour2013first}.
Then, we analyse the complexity of LJT stepwise and as a whole.
Last, we compare the complexity results with LVE and JT.

\paragraph{Liftability}
\cite{taghipour2013first} show that the FO dtree of a model allows for a liftability test for a lifted inference solution.
\begin{theorem}[\cite{taghipour2013first}] \label{prop:liftable}
	An FO dtree $T$ has a lifted inference solution if its clusters only consist of PRVs with representative constants and PRVs with one logvar.
	$T$ is called \emph{liftable} and its one-logvar PRVs are count-convertible.
\end{theorem}
Count-converting the one-logvar PRVs leads to a \emph{counted liftable FO dtree} of PRVs with representative constants and CRVs.
In a counted liftable FO dtree, all eliminations are lifted.
The FO dtree of $G_{ex}$ in \cref{fig:fodt} has only clusters of PRVs with representative constants and one-logvar PRVs.
That is a lifted solution for $G_{ex}$ is possible, which we have seen during the example calculations.
For the complexity results, we concentrate on models with counted liftable FO dtrees as these models have a lifted solution.

\paragraph{Complexity of LVE}
Given the lifting setup, \citet{taghipour2013first} introduce a lifted width to accommodate lifted calculations, which is defined as follows:
\begin{definition}
	The \emph{lifted width} $w_T$ of an FO dtree $T$ is a pair $(w_g, w_\#)$, $w_g$ is the largest ground width and $w_\#$ the largest counting width among the clusters of $T$.
\end{definition}
The largest ground width is the largest number of PRVs with representative objects in any cluster in $T$.
The largest counting width is the largest number of CRVs in any cluster in $T$.
Per cluster, the complexity depends on the largest possible size a factor can have, which depends on the largest range of its PRVs and CRVs as well as how many of PRVs and CRVs there are, i.e., $w_g$ and $w_\#$.
\begin{theorem}[\cite{Taghipour2013lifted}] \label{prop:complex:lve}
	In a counted liftable FO dtree $T$ of a model $G$, the \emph{node complexity} is 
	\begin{align}
		O(\log_2 n \cdot r^{w_g} \cdot n_{\scriptscriptstyle\#}^{w_\# \cdot r_\#}), \label{eq:complex:node}
	\end{align}
	with $(w_g, w_\#)$ being the lifted width of $T$, $n$ the largest domain size in $lv(G)$, $r$ the largest range size among the PRVs in $T$, $n_{\scriptscriptstyle\#}$ the largest domain size among the counted logvars, and $r_{\scriptscriptstyle\#}$ the largest range size among the PRVs in the CRVs.
	
	Let $n_T$ be the number of nodes in $T$.
	Then, the complexity of LVE is 
	\begin{align}
		O(n_T \cdot \log_2 n \cdot r^{w_g} \cdot n_{\scriptscriptstyle\#}^{w_\# \cdot r_\#}). \label{eq:complex:lve}
	\end{align}
\end{theorem}
The product $r^{w_g}\cdot n_{\scriptscriptstyle\#}^{w_\# \cdot r_\#}$ bounds the worst case size of a parfactor in a parcluster with $w_g$ PRVs and $w_\#$ CRVs, which bounds the number of summations in the factor.
The term $r^{w_g}$ refers to the size of the range of the PRVs with representative objects.
The term $n_{\scriptscriptstyle\#}^{w_\# \cdot r_\#}$ refers to the size of the range of the $w_\#$ CRVs.
The term $n_{\scriptscriptstyle\#}^{r_\#}$ over-approximates the range size of a CRV, which is given by $\binom{n_{\#} + r_{\#} -1}{n_{\#} - 1}$.
For computing an exponentiation after an elimination or for computing a count conversion, i.e., a number of exponentiations and multiplications, the factor of $\log_2 n$ multiplies into the largest possible factor size.
Given \cref{eq:complex:node}, the complexity of LVE in \cref{eq:complex:lve} follows.

\paragraph{Complexity of LJT}
Assume that we have a minimal FO jtree $J =(V,E)$ from a counted liftable FO dtree $T$ for a model $G$.
The effort for constructing $J$ from $T$ is in $O(n_T + n_J)$, $n_T$ being the number of nodes in $T$ and $n_J$ being the number of nodes in $J$ after minimising but before fusing.
Setting clusters as parclusters and minimising the result each visits all nodes in $T$, leading to a complexity of $O(n_T)$.
After minimisation, $J$ is a minimal FO jtree for $G$ with $n_J$ nodes.

Fusion guarantees that message calculations do not lead to groundings for a model that allows for a lifted solution.
LJT performs the fusion check for each edge twice.
The check itself as well as merging does not depend on the factor size, bounded by $w_g + w_{\scriptscriptstyle\#}$.
There are $n_J - 1$ edges, leading to a complexity of $O(n_J \cdot (w_g + w_{\scriptscriptstyle\#}))$.
Fusion may decrease $n_J$ and increase $w_J$.
As minimising $J$ leads to a smaller number of clusters, which may further decrease with fusion, $n_T$ is usually much larger than $n_J$.
So, one could replace $n_J$ with $n_T$.
After fusion, $J$ is a minimal FO jtree that does not induce groundings with $n_J = |V|$ nodes.
The notion of a lifted width also applies to FO jtrees, with $w_J = (w_g, w_\#)$ where $w_g$ is the largest number of PRVs in any parcluster of $J$ and $w_\#$ is the largest number of CRVs in any parcluster of $J$.

Analogously to \cref{prop:complex:lve}, the largest possible factor in $J$ is given by $r^{w_g} \cdot n_{\scriptscriptstyle\#}^{w_\# \cdot r_\#}$.
\emph{Evidence entering} consists of absorbing evidence at each node if applicable.
\begin{lemma}\label{prop:complex:ev}
	The complexity of absorbing an evidence parfactor is
	\begin{align}
		O(n_J \cdot \log_2 n \cdot r^{w_g} \cdot n_{\scriptscriptstyle\#}^{w_\# \cdot r_\#}). \label{eq:complex:ev}
	\end{align}
\end{lemma}

\emph{Passing messages} consists of calculating messages with LVE.
\begin{lemma}\label{prop:complex:msg}
	The complexity of passing messages is
	\begin{align}
		O(n_J \cdot \log_2 n \cdot r^{w_g} \cdot n_{\scriptscriptstyle\#}^{w_\# \cdot r_\#}). \label{eq:complex:msg}
	\end{align}
\end{lemma}

\begin{lemma}\label{prop:complex:qa}
	The complexity of answering a set of queries  $\{Q_k\}_{k=1}^m$ is
	\begin{align}
		O(m \cdot \log_2 n \cdot r^{w_g} \cdot n_{\scriptscriptstyle\#}^{w_\# \cdot r_\#}). \label{eq:complex:qa}
	\end{align}
\end{lemma}

\begin{theorem}\label{prop:complex:ljt}
	The complexity of LJT is
	\begin{align}
		O((n_J + m) \cdot \log_2 n \cdot r^{w_g} \cdot n_{\scriptscriptstyle\#}^{w_\# \cdot r_\#}). \label{eq:complex:ljt}
	\end{align}
\end{theorem}

\end{document}